\documentclass{article}
\usepackage{natbib}

\usepackage{enumitem}

\usepackage[final]{nips_2017}

\usepackage[utf8]{inputenc} 
\usepackage[T1]{fontenc}    
\usepackage{hyperref}       
\usepackage{url}            
\usepackage{booktabs}       
\usepackage{amsfonts,bm}       
\usepackage{nicefrac}       
\usepackage{microtype}      
\usepackage{amsmath}
\usepackage{amssymb}
\usepackage{algorithmic}
\usepackage{mathtools}
\usepackage{subcaption}
\usepackage{multirow}
\usepackage[usenames, dvipsnames]{color}
\usepackage[titlenotnumbered ,ruled]{algorithm2e}
\usepackage{caption} 
\newcommand{\E}{\mathbb{E}}

\newcommand{\e}{\mathrm{e}}

\DeclareMathOperator*{\argmin}{argmin}

\newenvironment{proof}{\paragraph{Proof:}}{\hfill$\square$}

\newtheorem{corollary}{Corollary}
\newtheorem{lemma}{Lemma}

\newtheorem{remark}{Remark}
\newtheorem{theorem}{Theorem}

\newtheorem{definition}{Definition}

\title{Stochastic Optimization with Bandit Sampling}

\author{
	Farnood Salehi  \And  L. Elisa Celis \And Patrick Thiran \AND	
	\textnormal{EPFL, Switzerland}\\
	\texttt{firstname.lastname@epfl.ch} \AND
}

\begin{document}
\setlength{\abovedisplayskip}{2pt}
\setlength{\belowdisplayskip}{2pt}

\maketitle

\vspace{-.25in}

\begin{abstract}
\vspace{-.1in}
    Many stochastic optimization algorithms work by estimating the gradient of the cost function on the fly by sampling datapoints uniformly at random from a training set.
However, the estimator might have a large variance, which inadvertantly slows down the convergence rate of the algorithms. 
One way to reduce this variance is to  sample the datapoints from a carefully selected non-uniform distribution. 
%
In this work, we propose a novel  non-uniform sampling approach that uses the multi-armed bandit framework. 
Theoretically, we show that our algorithm asymptotically approximates the optimal variance within a factor of 3.
Empirically, we show that using this datapoint-selection technique results in a significant reduction of the convergence time and variance of several stochastic optimization algorithms such as SGD and SAGA.
This approach for sampling datapoints is general, and can be used in conjunction with \emph{any} algorithm that uses an unbiased gradient estimation -- we expect it to have broad applicability beyond the specific examples explored in this work. 
\end{abstract}

\vspace{-.1in}
\section{Introduction} \label{sec:Introduction}
\vspace{-.1in}
Consider the following optimization problem that is ubiquitous in machine learning:
\begin{equation} \label{eq:main}
\min_{w\in\mathbb{R}^d} F(w) := \frac{1}{n}\sum_{i=1}^{n}\phi_i(w)+\lambda r(w),
\end{equation}
where the coordinates $w\in \mathbb{R}^d$ are the learning parameters.  The first term in \eqref{eq:main} (which we refer to as  $f(w)$) is the mean of $n$ convex functions 
$\phi_i(\cdot): \mathbb{R}^d \rightarrow \mathbb{R}$, called \emph{sub-cost functions}, while the second is the product of a convex regularizer $r(\cdot)$ and a regularization parameter $\lambda$. 
The $i^{th}$ sub-cost function $\phi_{i}(\cdot)$ is parameterized by the $i^{th}$ \emph{datapoint} $(x_i,y_i)$, where $x_i\in \mathbb{R}^d$ denotes its feature vector and $y_i \in \mathbb{R}$ its label. 
Examples of common sub-cost functions include
\begin{itemize}[leftmargin=15pt,noitemsep,topsep=0pt,parsep=0pt,partopsep=0pt]
	\item $L_1$-penalized logistic regression: $\phi_i(w) = \log (1+\exp(-y_i\langle x_i,w\rangle))$ and $r(w) = \|w\|_1$,
	\item $L_2$-penalized SVM: $\phi_i(w) =  ([1-y_i\langle x_i,w\rangle)]_+)^2$ and $r(w) =\frac{1}{2} \|w\|^2_2$ (where $[\cdot]_+ = \max\{0,\cdot\}$ is the hinge loss),
	\item Ridge regression:  $\phi_i(w) = \frac{1}{2}(\langle x_i,w\rangle - y_i)^2$ and $r(w) = \frac{1}{2}\|w\|_2^2$.
\end{itemize}

Gradient descent and its variants form classic and often very effective methods for solving \eqref{eq:main}. 
However, if we minimize $F(w)$ using gradient descent, 
each iteration needs $n$ gradient calculations (at iteration $t$, the value $\nabla\phi_i(w^t)$ must be computed for all $1\leq i \leq n$) which, for large $n$, can be prohibitively expensive \cite{B2010}. 
Stochastic gradient descent (SGD) reduces the computational complexity of an iteration by sampling a datapoint $i_t \in 1, \ldots, n$ uniformly at random at each time step $t$ and computing the gradient only at this datapoint;  
$\nabla \phi_{i_t}(w^t)$ is then an unbiased estimator for $\nabla f(w^t)$.
However, this estimator may have a large variance, which negatively affects the convergence rate of the underlying optimization algorithm and requires an increased number of iterations. 
For two classes of stochastic optimization algorithms, SGD and proximal SGD (PSGD), reducing this variance improves  the speed of convergence to the optimal coordinate $w^\star$ \cite{ZZ2014} (see also Section~\ref{sec:priminilary}).

This has motivated the development of several techniques to reduce this variance by using previous information to refine the estimation for the gradient; e.g., by occasionally calculating and using the full gradient to refine the estimation \cite{AY2016,XZ2014}, or using the previous calculations of $\phi_{i}$ (at the most recent selection of each datapoint $i$) \cite{DBL2014}. 
Yet another technique, closely related to this work, is to sample $i_t$ from a non-uniform distribution $p^t=[p^t_1,\cdots,p^t_n]$ (see \cite{KG2016,ZZ2014,zz2015,ZZ2014_2,ZKM2017,SBAD2015,CR2016}), where the probability $p^t_i$ of sampling datapoint $i$ at time $t$ is proportional to
 $\|\nabla \phi_i(w^t)\|$\footnote{In this work, we denote the Euclidian norm $\|\cdot\|_2$ by $\|\cdot\|$.}.
For example, in SGD with non-uniform sampling according to $p^t$, the update rule is
\begin{equation} \label{eq:SGD_update}
w^{t+1} = w^t - \gamma_t \left(\hat{g}(w^t) +\lambda \nabla r(w^t) \right),
\end{equation}
where $\gamma_t$ is the step size and $\hat{g}(w^t)$ is the unbiased estimator for 
$\nabla f(w^t)$ at time $t$ defined by
\begin{equation} \label{eq:SGD_gradient}
\hat{g}(w^t) \triangleq \frac{\nabla \phi_{i_t}(w^t)}{np^t_{i_t}}.
\end{equation}
  Taking expectation over $p^t$, the \emph{pseudo-variance}
  \footnote{
  Note that $\hat{g}(w^t)$ is a $d$-dimensional random vector, with $d>1$ in general, hence strictly speaking \eqref{eq:pseudo-variance} is the sum of the variances of its $d$ entries. Although \eqref{eq:pseudo-variance} is simply called varaince in  \cite{ZZ2014}, we use the term pseudo-variance for \eqref{eq:pseudo-variance} to distinguish it from the term variance.
} 
  of $\hat{g}(w^t)$ is defined to be
  \begin{equation} \label{eq:pseudo-variance}
  \begin{aligned}
  \mathbb{V}^t\left(w^t,p^t\right) & \triangleq \E\left[\left\|\hat{g}(w^t) - \nabla f(w^t)\right\|^2\right] = \E\left[\left\|\frac{1}{np^t_{i_t}}\nabla \phi_{i_t}(w^t) - \frac{1}{n}\sum_{i=1}^{n}\nabla \phi_i(w^t)\right\|^2\right]. 
  \end{aligned}
  \end{equation}
  Expanding \eqref{eq:pseudo-variance}, one can write $\mathbb{V}^t\left(w^t,p^t\right)$ as the difference of two terms. The first is a function of $p^t$, which we refer to as the \emph{effective variance}
  	\begin{equation} \label{eq:effective-variance}
  \mathbb{V}_e^t\left(w^t,p^t\right) \triangleq \frac{1}{n^2} \sum_{i=1}^{n}\frac{1}{p_i^t}\left\|\nabla \phi_i(w^t)\right\|^2,
  \end{equation}
while the second does not depend on $p^t$, and we denote it by 
$  \mathbb{V}_c^t(w^t) \triangleq \nicefrac{\left\|\sum_{i=1}^{n}\nabla \phi_i(w^t)\right\|^2}{n^2}$.
As the only term in one's control is $p^t$, it suffices to minimize $\mathbb{V}_e^t\left(w^t,p^t\right)$:  
 The minimum of \eqref{eq:effective-variance} and thus \eqref{eq:pseudo-variance} is attained when $p^t_i = \nicefrac{\|\nabla \phi_i(w^t)\|}{\left(\sum_{j=1}^{n}\|\nabla \phi_j(w^t)\|\right)}$.
If the $\nabla \phi_i(w^t)$s have similar magnitudes for all $1\leq i \leq n$, then $p^t$ is close to the uniform distribution. However, if the magnitude of $\nabla \phi_{i}(w^t)$ at some datapoint $i$ is comparatively very large, then the optimal distribution is far from uniform; in this case, the optimal effective variance can be roughly $n$ times smaller than the effective variance using the uniform distribution.
  \emph{How do we find the optimal probabilities $p^t$, given that the gradients $\nabla \phi_{i}(w^t)$ are unknown?}
 In \cite{KG2016,ZZ2014,zz2015} the question is approached by minimizing an upper bound on $\mathbb{V}_e^t\left(w^t,p^t\right)$, which results in a time-invariant distribution $p^t = p$ for all $t$. This method is known as \emph{importance sampling} (IS).
 However, a drawback of this method is that the upper-bound on \eqref{eq:effective-variance} may be loose and hence far from the optimal distribution. Moreover, this requires the computation of an upper-bound on $\|\nabla \phi_i(w^t)\|$ for all $1\leq i \leq n$, which can be computationally expensive.

\vspace{-.1in}
 \subsection*{Our Contributions}
\vspace{-.1in}
In this work, inspired by active learning methods, we use an adaptive approach to define $p^t$ instead of fixing it in advance. 
 If the set of datapoints selected during the first $\ell$ iterations is $\{i_t\}_{1\leq t \leq \ell}$, then we refer to the corresponding gradients $\{\nabla \phi_{i_t}(w^t)\}_{1\leq t \leq \ell}$ as \emph{feedback}, and use it define $p^{\ell+1}$.
 The problem of how to best define the distribution given the feedback falls under the framework of multi-armed bandit problems.
 We call our approach \emph{multi-armed bandit sampling} (MABS) and  show that this approach gives a  distribution that is asymptotically close to optimal. 
\vspace{-.1in}
\begin{theorem}[Informal Statement of Theorem~\ref{thm:main}]
	Let $ p^\star = \mathop{\argmin}_{p} \sum_{t = 1}^{T} \mathbb{V}_e^t\left(w^t,p\right)$ be the (a-priori unknown) distribution that optimizes the effective variance after $T$ iterations. 
	Let $p^t$  be the distributions selected by MABS. 
	When the gradients are bounded, MABS approximates the optimal solution $\sum_{t = 1}^{T} \mathbb{V}_e^t\left(w^t,p^\star\right)$ asymptotically up to a factor 3, i.e.,
	\begin{equation}
	\lim\limits_{T \to \infty} \frac{1}{T}\left( \sum_{t = 1}^{T} \left(\mathbb{V}_e^t\left(w^t,p^t\right) - 3\mathbb{V}_e^t(w^t,p^\star) \right) \right) \leq 0. 
	\end{equation}
	\end{theorem}
	%
%
We emphasize that MABS can be used in conjunction with \textit{any} algorithm that uses an unbiased gradient estimation to reduce the variance of estimation, not just SGD. This includes SAGA, SVRG, Prox\_SGD, S2GD, and Quasi\_Newton methods. We present the empirical performance of some of these methods in the paper (see Figures \ref{fig:SGD_main1}, \ref{fig:main} and \ref{fig:gamma}).

In summary, our main contributions are:
\begin{itemize}[leftmargin=15pt,noitemsep,topsep=0pt,parsep=0pt,partopsep=0pt]
	\item Recasting the problem of reducing the variance of stochastic optimization  as a multi-armed bandit problem as above, 
	\item Providing a sampling algorithm (MABS) and an analysis of its rate of convergence the optimal distribution (Section~\ref{sec:bandit}).
	\item Illustrating the convergence rates of stochastic optimization algorithms, such as SGD, when combined with MABS (Section~\ref{sec:priminilary}).
	\item Exhibiting the significant improvements in practice yielded by selecting datapoints using MABS for stochastic optimization algorithms on both synthetic and real-world data (Section~\ref{sec:experiment}).
\end{itemize}

\section{Multi-Armed Bandit Sampling} \label{sec:bandit}

The end goal of MABS is to find the sampling distribution $p^t$ that minimizes the effective variance $\mathbb{V}_e^t\left(w^t,p^t\right)$, and thus the pseudo-variance $\mathbb{V}^t\left(w^t,p^t\right)$. 
 In SGD and other stochastic optimization algorithms that use $\hat{g}(w^t) = \nicefrac{\nabla \phi_{i_t}(w^t)}{np^t_{i_t}}$ as an unbiased estimator for $\nabla f(w^t)$, the effective variance $\mathbb{V}_e^t\left(w^t,p^t\right) = \mathbb{E}[\|\hat{g}(w^t)\|^2]$. However, we consider a broader class of stochastic optimization algorithms, for which
  \begin{equation} \label{eq:pseudo-variance_general}
\begin{aligned}
\mathbb{V}^t\left(w^t,p^t\right) & \triangleq \E\left[\left\|\hat{g}(w^t) - \nabla f(w^t)\right\|^2\right] = \mathbb{V}^t_e(w^t,p^t) - \mathbb{V}^t_c(w^t),
\end{aligned}
\end{equation}
 where $\mathbb{V}^t_c(w^t)$ does not depend on $p^t$, and where the effective variance has the form, dropping this explicit dependence on $w^t$,
 \begin{equation} \label{eq:Bandit}
 \mathbb{V}_e^t\left(p^t\right) = \sum_{i=1}^{n}\frac{a_i^t}{p_i^t}, 
\end{equation}
where $a^t_i$ is a function of the coordinate $w^t$ and of the estimator (for the gradient) that is used. 
For example, for SGD $\hat{g}(w^t)$ is given by \eqref{eq:SGD_gradient} and $a_i^t = \nicefrac{\|\nabla \phi_{i}(w^t)\|^2}{n^2}$.
Let $p^\star$ denote the invariant distribution at which  $\sum_{t=1}^{T} \mathbb{V}_e^t\left(p^t\right)$ reaches its minimum, i.e., 
\begin{equation} \label{eq:optimal_p}
p^\star = \argmin_{p}\sum_{t=1}^{T} \mathbb{V}_e^t\left(p\right). 
\end{equation}
The goal is to find an approximate solution of $\sum_{t=1}^{T} \mathbb{V}_e^t\left(p^\star\right)$, i.e., a distribution $p^t$ for all $1\leq t \leq T$ such that $\sum_{t=1}^{T} \mathbb{V}_e^t\left(p^t\right) \leq c \sum_{t=1}^{T} \mathbb{V}_e^t\left(p^\star\right)$  for some $c>1$.
We first present a technical result that motivates the use of MAB in this setting.
\begin{lemma} \label{lem:bound}
For any real value constant $\zeta \leq 1$ and sampling distributions $p^1$ and $p^2$ we have
		\begin{equation}\label{eq:transform1}
	(1-2\zeta) \mathbb{V}_e^t\left(p^1\right) - (1-\zeta) \mathbb{V}_e^t\left(p^2\right) \leq \langle p^1 - p^2, \nabla \mathbb{V}_e^t\left(p^1\right) \rangle + \zeta \langle p^2,\nabla \mathbb{V}_e^t\left(p^1\right) \rangle.
	\end{equation}
\end{lemma}
The proof is in Appendix~\ref{sec:appendix}. It is based on the convexity property of $\mathbb{V}_e^t\left(p\right)$ with respect to $p$. 
Let $p^1=p^t$ and $p^2=p^\star$ in \eqref{eq:transform1}, then
\begin{equation}\label{eq:transform}
(1-2\zeta) \mathbb{V}_e^t\left(p^t\right) - (1-\zeta) \mathbb{V}_e^t\left(p^\star\right) \leq \langle p^t - p^\star, \nabla \mathbb{V}_e^t\left(p^t\right) \rangle + \zeta \langle p^\star,\nabla \mathbb{V}_e^t\left(p^t\right) \rangle.
\end{equation}

A MAB has $n$ arms (which are the $n$ datapoints in our setting). Selecting arm $i$ at time $t$ gives a negative reward (loss) $r^t_i$ and losses vary among arms. 
 At time $t$, a MAB algorithm updates the arm sampling distribution $p^t$ based on the loss $r^t_i$ of the arm $i$ that is selected at time $t$, but has no access to the losses $r_j^t$ of other arms $j\neq i$.
In our setting, we update the sampling distribution $p^t$ based on the $a^t_i$ computed from sampled gradient $\nabla \phi_{i}(w^t)$.
The probability of selecting an arm $i$ at time $t$ is $p^t_i$. Let $p^\star$ be the optimal distribution that minimizes the cumulated loss over $T$ rounds. Then
$C^t = \langle p^t-p^\star, r^t \rangle$ is the cost function at time $t$ that one wants to minimize. 
Now, observe that the first term in the right-hand side of \eqref{eq:transform} is the cost function $C^t$ of an adversarial MAB,
 where $r^t_i = \nabla_i \mathbb{V}_e^t(p^t) = -\nicefrac{a^t_i}{(p^t_i)^2}$ , where $p^t$ is the arm/datapoint distribution at time $t$, and $p^\star$ is the optimal sampling distribution given by \eqref{eq:optimal_p}.
 
 Building on this analogy between MAB and datapoint sampling, we propose the MABS algorithm, based on EXP3 \cite{ACFS2002}. The MABS algorithm has $n$ weights $\{\text{w}_i^t\}_{1\leq i\leq n}$, each initialized to 1. 
The sum of weights is called potential function $W^t = \sum_{j=1}^{n}\text{w}^{t}_j$.
The distribution $p^t$ is a weighted average between the distribution $\{\nicefrac{\text{w}^{t}_i}{W^t}\}_{1\leq i\leq n}$ at time $t$ and the uniform distribution $\{\nicefrac{1}{n}\}_{1\leq i \leq n}$, i.e., $p_i^t \propto (1-\eta)\text{w}^{t}_i+\eta$. The parameter $\eta$ determines how much $p^t_i$ deviates from the uniform distribution.  MABS updates the weight of selected datapoint $i_t$ at time $t$, according to the updating rule $\text{w}^{t+1}_{i_t} = \text{w}^t_{i_t}  \exp(\delta
\nicefrac{a_{i_t}^t}{(p^t_{i_t})^3)}$ and keeps the others fixed, i.e., $\text{w}_i^{t+1}=\text{w}_i^t$ for all $i \neq i_t$.
\begin{algorithm}[t]
	\caption{MABS}
	\label{alg:Bandit}
	\begin{algorithmic}[1]
		\STATE \textbf{initialize: }  $\eta = 0.4$ and $\delta=\sqrt{\nicefrac{\eta^4\ln n}{(Tn^5\overline{(a^2)})}}$
		\STATE \textbf{initialize: }  $p_i^1 = \nicefrac{1}{n}$, $\text{w}_i^1=1$, \hspace{1mm} for all $1\leq i \leq n$
		
		\For{$t=1:T$}{
			sample $i\sim p^t$ \\
			$\text{w}^{t+1}_i = \text{w}^t_i \cdot \exp(\frac{\delta a_i^t}{(p^t_i)^3})$ \\
			$\text{w}^{t+1}_j = \text{w}^t_j$, \hspace{2.8cm} for all $ j \neq i$\\
			$W^{t+1} = \sum_{j=1}^{n}\text{w}^{t+1}_j$ \\
			$p^{t+1}_j \leftarrow (1-\eta)\frac{\text{w}^{t+1}_j}{W^{t+1}}+\frac{\eta}{n}$, \hspace{0.65cm} for all $1\leq j \leq n$ \\
		}
	\end{algorithmic}
\end{algorithm}

\begin{remark}
	A  difference between the variance-reduction problem and MAB is that
	 in MAB the rewards are assumed to be upper bounded almost surely. However, here the rewards might be unbounded, depending on the distribution $p^t$.  This occurs if the probability $p^t_i$ is close to 0, so making the term $\nicefrac{a^t_i}{(p^t_i)^2}$ very large. 
\end{remark} 
\begin{theorem}\label{thm:main}
 Using MABS with $\eta = 0.4$ and $\delta=\sqrt{\nicefrac{\eta^4\ln n}{(Tn^5\overline{(a^2)})}}$ to minimize \eqref{eq:Bandit} with respect to $\{p^t\}_{1\leq t \leq T}$, we have
	\begin{equation}\label{eq:performance1}
	\sum_{t=1}^{T} \mathbb{V}_e^t(p^t)  \leq 3 \sum_{t=1}^{T} \mathbb{V}_e^t(p^\star) +
	50 \sqrt{n^5T\overline{(a^2)} \ln n} ,
	\end{equation}
	where $T \geq 25n\ln n \cdot \nicefrac{\max_i(a_i)^2}{(4 \overline{(a^2)})} $, for some $a_i \geq \sup_t \{a^t_i\}$, and where $\overline{(a^2)} = \sum_{i=1}^{n}\nicefrac{a_i^2}{n}$. The complexity of MABS is $O(T\log n)$.
\end{theorem}

The proof is given in Appendix~\ref{sec:appendix}. To show that MABS minimizes asymptotically the effective variance $\mathbb{V}^t_e$ in Theorem~\ref{thm:main}, we adapt the approach of multiplicative-weight update algorithms (see for example \cite{ACFS2002}), using the results of Lemma~\ref{lem:bound}:
We upper bound and lower bound the potential function $W^T$ at iteration $T$, and then use Lemma~\ref{lem:bound} to upper-bound  $\sum_{t=1}^{T} \mathbb{V}_e^t(p^t) $.

Although the second term of the right-hand side of \eqref{eq:performance1} increases as $n^{5/2}$, the effective variance $\mathbb{V}_e^t(p^t)$ scales as $\sqrt{n}$ because $a_i$ decreases as $\nicefrac{1}{n^2}$, hence $\mathbb{V}_e^t(p^t)$ increases only as $\sqrt n$.
 Note that to run MABS, only an upper bound on $\sum_{i=1}^{n}\nicefrac{a_i^2}{n}$
 is  needed, hence  we do not need to compute $\sup_t \{a^t_i\}$ exactly, whereas in IS the exact value of $\sup_t \{a^t_i\}$ is required. The computation of the gradient $\nabla \phi_{i}(w^t)$ requires $O(d)$ computations, so that the computational overhead of MABS is insignificant only if  $\log n $ is small compared to the coordinate dimension $d$.
 Such is the case for the two datasets in Table~\ref{table:stat} used in the evaluation section (Section~\ref{sec:experiment}). The condition on $T$ might be prohibitive if $n$ is large. However, we can relax this condition at the expense of having a slightly worse bound (see Appendix~\ref{sec:appendix}).

\begin{remark}
	If we know $a_i=\sup_t \{a^t_i\}$, then we can refine MABS and improve the bound \eqref{eq:performance1}.
	The idea is that,
	instead of mixing the distribution $\{\nicefrac{\text{w}_i^t}{W^t}\}_{1\leq i \leq n}$ with a uniform distribution, we mix $\{\nicefrac{\text{w}_i^t}{W^t}\}_{1\leq i \leq n}$ with a non-uniform distribution $\{ \nicefrac{a_i^{2/5}}{(\sum_{j=1}^{n}a_j^{2/5})}\}_{1\leq i \leq n}$, i.e.,  $p^t_i \propto (1-\eta)\text{w}^{t}_i+\eta a_i^{2/5}$ instead of $(1-\eta)\text{w}^{t}_i+\eta$ at the last line of MABS. This way we can improve the worst-case guarantee on $\nicefrac{a^t_i}{p^t_i}$, because the lower bound on $p^t_i$ is larger for a datapoint $i$ whose $a^t_i$ is large (see Appendix~\ref{sec:MABS_IS} for this variant of algorithm).
\end{remark}

\vspace{-.1in}
\section{Combining MABS with Stochastic Optimization Algorithms} 
\vspace{-.1in}
\label{sec:priminilary} \label{sec:applications}
In this section, we restate the known convergence guarantees for SGD and PSGD in order to highlight the impact the effective variance has on them.
As the upper-bounds on the convergence guarantees depend on the effective variance $\mathbb{V}^t_e(p^t)$, by using the sampling distribution $p^t$ given by MABS, the effective variance $\mathbb{V}^t_e(p^t)$ is reduced, which results in improved convergence guarantees.
Recall that these algorithms use an unbiased estimator for the gradient $\hat{g}(w^t) = \nicefrac{\nabla \phi_{i_t}(w^t)}{np_{i_t}}$ for $f(w)$, hence $\mathbb{V}^t_e(p^t) = \E[\|\hat{g}(w^t)\|^2]$. 

\vspace{-.05in}
\subsection{SGD}
\vspace{-.05in}
For SGD, the known convergence rate can be restated in terms of the effective variance as follows.
\begin{theorem} [Theorem 1.17 in \cite{V2015}]\label{thm:SGD_strong}
	Assume that $F(w)$ is $\mu$-strongly convex. Then, if $\gamma_t = 2/\mu t$ in \eqref{eq:SGD_update}, the following inequality holds 
	for any $T \geq 1$ in SGD:
	\begin{equation} \label{eq:convergence_SGD}
	\mathbb{E}\left[F\left(\frac{2}{T(T+1)}\sum_{t=1}^{T}t\cdot w^{t}\right)\right]  - F(w^\star) \leq \frac{2}{\mu T(T+1)}\sum_{t=1}^{T} \E[\mathbb{V}^t_e(p^t)].
	\end{equation}
	The expectation is over $w^t$.
\end{theorem}
The convergence bound \eqref{eq:convergence_SGD} holds for any $p^t$ including the one given by MABS. Next, we consider 
 SGD in conjunction with with MABS and want to restate \eqref{eq:convergence_SGD} by plugging the upper-bound \eqref{eq:performance1} in it. 
 \vspace{-.1in}
 \begin{corollary}
 	Assume that $F(w)$ is $\mu$-strongly convex. Then, if $\gamma_t = 2/\mu t$ in \eqref{eq:SGD_update}, the following inequality holds 
 	for any $T \geq 25n\ln n \cdot \nicefrac{\max_i(a_i)^2}{(4 \overline{(a^2)})} $ in SGD with MABS:
 	\begin{equation} \label{eq:convergence_SGD_MABS}
 	\mathbb{E}\left[F\left(\frac{2}{T(T+1)}\sum_{t=1}^{T}t\cdot w^{t}\right)\right]  - F(w^\star) \leq \frac{2}{\mu T(T+1)} \left(3\E[\mathbb{V}^\star(p^\star)] + 50\sqrt{T\sum_{i=1}^{n}a_i^2\ln n} \right),
 	\end{equation}
 	for some $a_i \geq \sup_w \|\nabla \phi_{i}(w)\|^2$,
 	$p^\star$ is given by \eqref{eq:optimal_p}, where $\mathbb{V}^\star(p^\star) = \sum_{t=1}^{T}\mathbb{V}^t_e(p^\star
 	)$ is the optimal pseudo-variance cumulated over $T$ iterations.
 \end{corollary}

Note that in $3\E[\mathbb{V}^\star(p^\star)]+50\sqrt{T\sum_{i=1}^{n}G_i^4\ln n}$,
the first term increases as $T$ and the second term increases as $\sqrt{T}$, which means that asymptotically $\sqrt{T\sum_{i=1}^{n}G_i^4\ln n}/\E[\mathbb{V}^\star(p^\star)] \to 0$ as $T\to \infty$. Meaning that $\E[\mathbb{V}^\star(p^\star)]$ is dominant in convergence guaranty \eqref{eq:convergence_SGD_MABS} for large $T$.
 In SGD with uniform sampling $p^t_i = p_i = \nicefrac{1}{n}$, however the effective variance $\sum_{t=1}^{T} \E[\mathbb{V}^t_e(p)]$ can be much larger than $\E[\mathbb{V}^\star(p^\star)]$, hence the convergence guaranty \eqref{eq:convergence_SGD} can be very poor comparing to \eqref{eq:convergence_SGD_MABS}.

When the effective variance $\mathbb{V}^t_e(p^t)$ is small (meaning that \eqref{eq:SGD_gradient} is a good estimator), we expect a more stable algorithm, i.e., we can choose larger step size $\gamma_t$ without diverging.
Assume that $F(w) = f(w) = \sum_{i=1}^{n} \phi_i(w)/n$, i.e., there is no reguralizer $\lambda=0$ in \eqref{eq:main} and it is  $L$-smooth. 
Using the smoothness property in \eqref{eq:smooth} where $h(\cdot)=F(\cdot)$, $y=w^{t+1}$ and $x = w^{t}$, we get 
\begin{equation} \label{eq:L_upper}
F(w^{t+1}) - F(w^{t}) \leq \langle \nabla F(w^t), w^{t+1}-w^t \rangle + L \|w^{t+1}-w^t\|^2,
\end{equation}
plugging the update rule \eqref{eq:SGD_update} of SGD 
\begin{equation}
\begin{aligned}
F(w^{t+1}) - F(w^{t})\leq -\gamma \langle \nabla F(w^t),\frac{\nabla \phi_{i_t}(w^t)}{np^t_{i_t}} \rangle +\gamma^2L \|\frac{\nabla \phi_{i_t}(w^t)}{np^t_{i_t}}\|^2 .
\end{aligned}
\end{equation}
Taking expectations, conditionally to $w^t$, we get $\E[F(w^{t+1})|w^t] - F(w^{t}) \leq -\gamma \| \nabla F(w^t)\|^2 + \gamma^2 L \sum_{i=1}^{n}\mathbb{V}^t_e(p^t)$.
To guarantee that the cost function decreases (in expectation), we need to have
$\gamma  \leq \| \nabla F(w^t)\|^2/\left(L \cdot\mathbb{V}^t_e(p^t)\right)$.
Therefore, by decreasing $\mathbb{V}^t_e(p^t)$, we can afford a larger step size $\gamma$. 
We test the stability of the SGD with MABS for a range of $\gamma$ in Section~\ref{sec:exp_robust} and show its significant stability compared to the SGD.

\vspace{-.05in}
\subsection{PSGD}
\vspace{-.05in}
For PSGD, let the function $f(w)$ be $\mu$-strongly convex and $L$-smooth with respect to $\psi$, a continuously differentiable function, and let $\mathcal{B}_\psi(w_1,w_2)$ be the Bregman divergence associated with the function $\psi$ (see Appendix~\ref{sec:definition} for a summary of these standard definitions).
PSGD updates $w$, according to
\begin{equation} \label{eq:PSGD_update}
w^{t+1} = \arg \min_w \left[ \langle \nabla \phi_{i_t}(w^t),w\rangle+\lambda r(w) +\frac{1}{\gamma_t} \mathcal{B}_\psi (w,w^t)
\right].
\end{equation}
Intuitively, this method works by minimizing the first-order approximation of the function $\phi_{i_t}$ plus the regularizer $\lambda r(w)$.
In the non-uniform version of this algorithm, $\nabla \phi_{i_t}(w^t)$ is replaced by $\nabla \phi_{i_t}(w^t)/(np^t_{i_t})$.
\begin{theorem}[Theorem 1 in \cite{ZZ2014}]\label{thm:zhang1}
	Assume that $\psi(\cdot)$ is $\sigma$-strongly convex, that $f(w)$ is $\mu$-strongly convex and $L$-smooth with respect to $\psi$, and that $r(w)$ is convex. Then, if 
	$\gamma_t=1/\left(\alpha + \mu t\right)$ in \eqref{eq:PSGD_update} with $\alpha \geq L-\mu$, the following inequality holds 
	for any $T \geq 1$ in PSGD:
	\begin{equation} \label{eq:PSGD_convergence}
	\frac{1}{T} \sum_{t=1}^{T} \mathbb{E}[F(w^{t+1})]  - F(w^\star) \leq \frac{\alpha}{T} \mathcal{B}_\psi(w^\star,w^1) + \frac{1}{T}\sum_{t=1}^{T} \frac{\E[\mathbb{V}^t_e(p^t)]}{\sigma(\alpha+\mu t)}.
	\end{equation}
\end{theorem}
Similar to SGD, \eqref{eq:PSGD_convergence} holds for any distribution $p^t$, hence we expect that by using MABS and having a small effective variance in \eqref{eq:PSGD_convergence}, PSGD can have a better convergence guarantee.
The study of convergence guarantee \eqref{eq:PSGD_convergence} when PSGD is used with MABS and its stability is left for future work.

\vspace{-.1in}
\section{Empirical Results} \label{sec:experiment} 
\vspace{-.1in}
We evaluate the performance of MABS in conjunction with several stochastic optimization algorithms and address the question: \emph{How much can our bandit-based sampling help?} 
Towards this, we compare the performance of several stochastic optimization algorithms that use MABS as compared with their UNIFORM or IMPORTANCE SAMPLING (IS) versions.
To do so, one must first define the appropriate unbiased estimator $\hat{g}(w^t)$ for $\nabla f(w^t)$ and $a_i^t$ (see \eqref{eq:Bandit}) for each algorithm. 
In particular, we compare the following algorithms and here present the necessary definitions for MABS: 
\begin{itemize}[leftmargin=15pt,noitemsep,topsep=0pt,parsep=0pt,partopsep=0pt]
\item \textbf{Stochastic Gradient Descent (SGD)}:

 $\hat{g}(w^t) = \nicefrac{\nabla \phi_{i_t}(w^t)}{(np^t_{i_t})} $ and $a_i^t = \nicefrac{\|\nabla \phi_{i}(w^t)\|^2}{n^2}$.
\item \textbf{Proximal Stochastic Variance-Reduced Gradient (Prox\_SVRG)}: 

$\hat{g}(w^t) = \nicefrac{\left(\nabla \phi_{i_t}(w^t) - \nabla \phi_{i_t}(\tilde{w})\right)}{(np^t_{i_t}) }+\nicefrac{\sum_{i=1}^{n}\nabla\phi_{i}(\tilde{w})}{n}$ and $a_i^t = \nicefrac{\|\nabla \phi_{i}(w^t)-\nabla \phi_{i}(\tilde{w})\|^2}{n^2}$, where $\tilde{w}$ is defined as follows. Time is divided into bins of size $n$, and $\tilde{w}$ is updated at the beginning of each bin. In the $c^{th}$ bin (i.e., if $cn\leq t < (c+1)n$), then $\tilde{w} = \nicefrac{\sum_{(c-1)n}^{cn} w^t}{n}$ (see \cite{XZ2014} for more details and \cite{AY2016,KG2016} for an improved version of the algorithm).

\item \textbf{SAGA}: $\hat{g}(w^t) = \nicefrac{\left(\nabla \phi_{i_t}(w^t) - \nabla \phi_{i_t}(\tilde{w}_{i_t})\right)}{(np^t_{i_t})} +\nicefrac{\sum_{i=1}^{n}\nabla\phi_{i}(\tilde{w}_i)}{n}$ and $a_i^t = \nicefrac{\|\nabla \phi_{i}(w^t)-\nabla \phi_{i}(\tilde{w}_i)\|^2}{n^2}$, where $\nabla\phi_{i}(\tilde{w}_i)$ is the gradient of sub-cost function $\phi_i$ at last time that datapoint $i$ was chosen (see \cite{DBL2014} for more details).

\end{itemize}
 For each stochastic optimization algorithm, we use three sampling methods: (1) uniform sampling (denoted by suffix \_U), (2) IS (denoted by suffix \_IS), and (3) MABS (denoted by suffix \_MABS).

\begin{figure}[t] 
	\centering
	\subcaptionbox{The cost $F$ for different $\tau$.
	\label{fig:delta_F}}[.4\linewidth][c]{%
		\includegraphics[width=.4\linewidth]{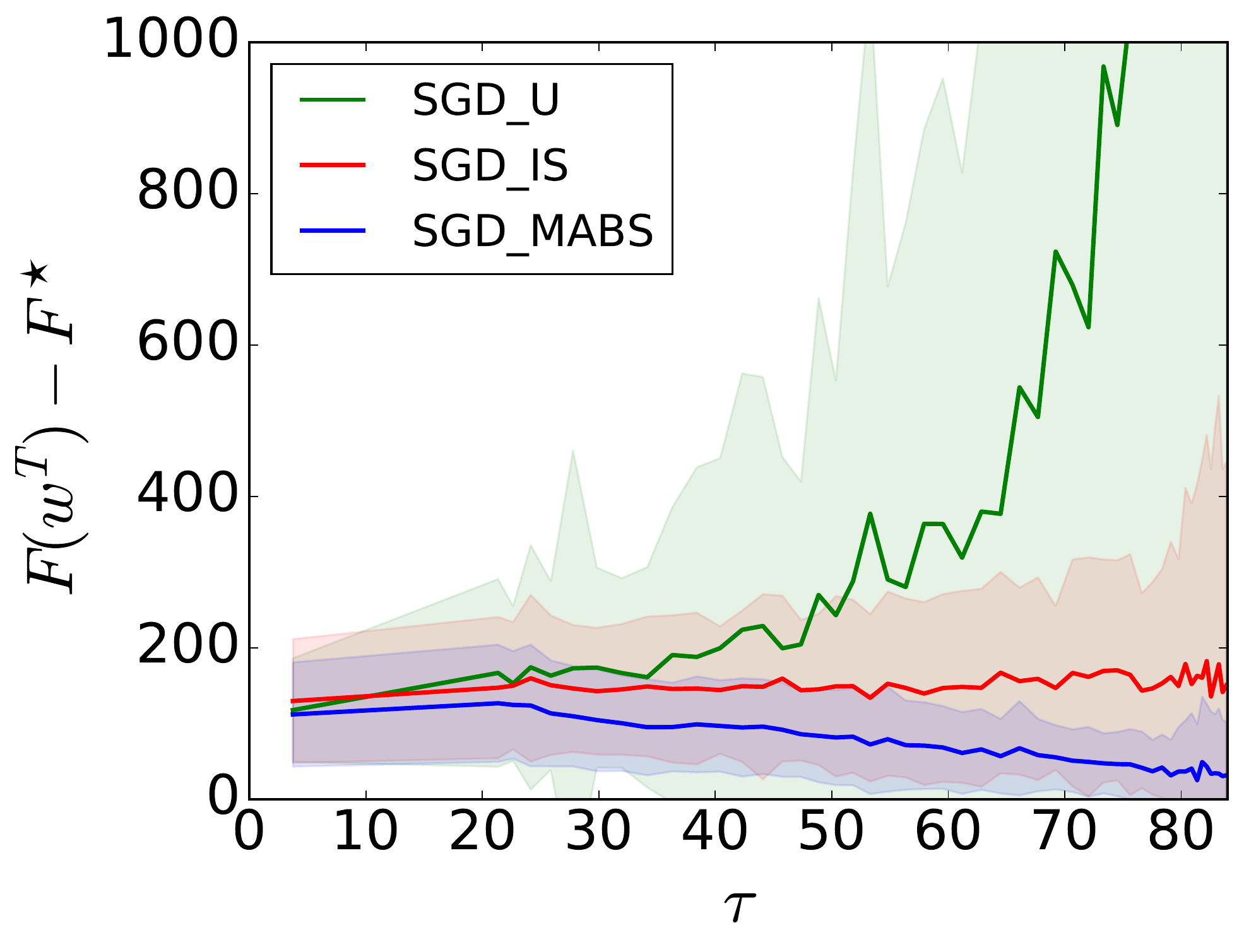}}
	\subcaptionbox{The effective variance $\mathbb{V}_e^T$ for different $\tau$.\label{fig:variance_tau}}[.4\linewidth][c]{%
		\includegraphics[width=.4\linewidth]{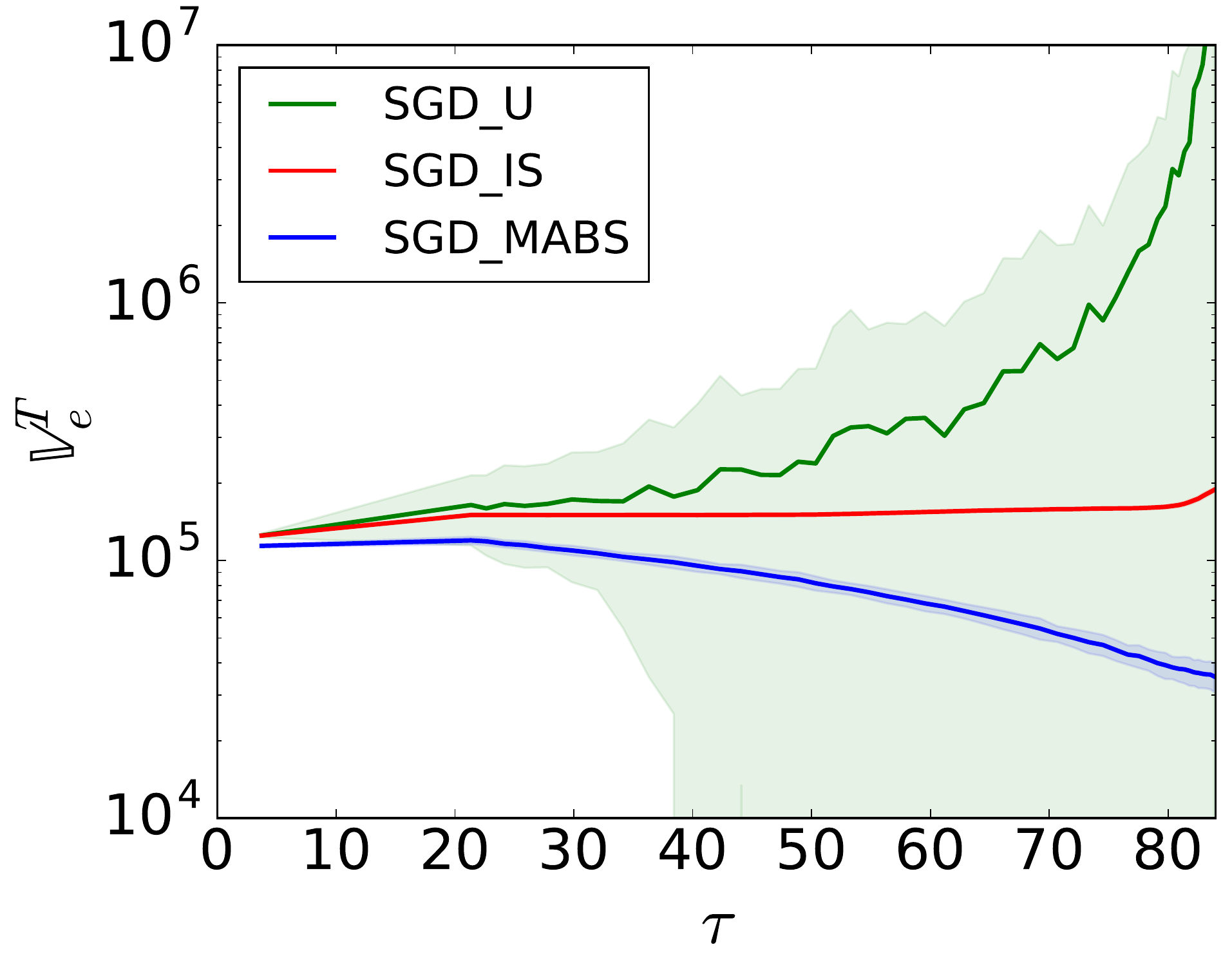}}
		\vspace{-.2in}	
		
	\caption{We study SGD with different sampling methods by comparing the convergence 
and the effective variance as a function of $\tau$ (a measure of the dissimilarity of the $\nabla \phi_i$s). We observe that both are lowest when MABS is used, and this effect increases significantly in $\tau$. \vspace{-.2in}}
\label{fig:SGD_main1}
\end{figure}

\vspace{-.1in}
\subsection{Empirical Results on Synthetic Data}
\vspace{-.1in}
As discussed in Section~\ref{sec:Introduction}, the benefit of MABS (and of non-uniform sampling more generally) will depend on how \emph{similar} the $\nabla \phi_i$s are. 
Let $L_i$ be the smoothness parameter of the sub-cost functions $\phi_i$, let $L_m = \max_i\{L_i\}$ be the maximum-smoothness, $\bar{L}= \sum_{i=1}^{n}\nicefrac{L_i}{n}$ be the average-smoothness, and $\tau = \nicefrac{L_m}{\bar{L}}$ be their ratio.
As observed in \cite{KG2016,ZZ2014}, when $\tau$ is large, we expect non-uniform sampling (and in particular MABS) to be more advantageous. 
To study this effect, we present results on synthetic datasets with different $\tau$ using  SGD, SGD\_IS, and SGD\_MABS.\footnote{In SGD\_IS, the sampling distribution is $p_i = L_i/(\sum_{j=1}^{n}L_j)$ (see \cite{ZZ2014}).} 

\vspace{-.1in}
\paragraph{Dataset.}
The datasets have $n=101$ datapoints and $d=5$ features.\footnote{Similar results are obtained for different values of $n$ and $d$.}
The labels are defined to be $y_i \triangleq \langle x_i,\beta\rangle + N_i$, where $\beta \in \mathbb{R}^5$ is the coefficient of the hyperplane generated from a Gaussian distribution with mean 0 and standard deviation 10, and $N_i$ is Gaussian noise with mean 0 and variance 1. 
The features $x_i \in \mathbb{R}^5$ are generated from a Gaussian distribution whose mean and variance are generated randomly.
 In order to obtain different $\tau$, we choose the datapoint $i$ with the largest smoothness $L_i=L_m$ and multiply its entire feature vector $x_i$ by a number $c>1$, whereas all labels and all other features remain fixed. 
 This increases $L_m$, and hence $\tau$.
The sub-cost function used here is $\phi_i(w) = \nicefrac{(\langle x_i,w\rangle-y_i)^2}{2}$, i.e., ridge regression with $\lambda = 0$.
All the algorithms use the same step size $\gamma=4\cdot10^{-3}$. 
Each experiment is run for $T = 3000$ iterations and repeated $k = 200$ times. 
We report the effective variance $\mathbb{V}^T_e(w^T)$ at iteration $T$, and the difference of values $F(w^T)$ found by three sampling versions of SGD and the value $F^\star$ found by gradient descent, to compare the stochastic algorithms (SGD, SGD\_IS, and SGD\_MABS) to the ideal gradient descent.

\vspace{-.1in}
\paragraph{Results.}
In Figure~\ref{fig:delta_F}, we observe that MABS has the best performance of all three sampling methods as the value of $F(w^T)$ for SGD\_MABS is the closest to $F^\star$ for all $\tau$. 
Additionally, as $\tau$ increases, the performance of SGD\_MABS further improves, confirming the intuition that when there is a datapoint with large gradient the convergence of MABS to the optimal sampling distribution is faster.
 On the other hand, as expected, the performance of SGD\_U degrades significantly in $\tau$. 
As SGD\_IS does not appear to be affected by $\tau$, the advantage of MABS over IS is strongest for large $\tau$. 
 Figure~\ref{fig:variance_tau} depicts the effective variance $\mathbb{V}^T_e(w^T)$ at final iteration $T$ as a function of $\tau$, and similar observations can be made. In particular, the effective variance of SGD\_MABS is lowest, and is decreasing in $\tau$ while the effective variance of SGD\_IS and SGD\_U are non-decreasing and increasing in $\tau$ respectively.

\vspace{-.1in}
\subsection{Empirical Results on Real-world Data}
\vspace{-.1in}
\begin{table}
	\caption{Statistics of the datasets used in the experiments.}
	\centering
\begin{tabular}{c | c c c}
	Dataset & $n$ & $d$ & $\tau$\\
	\hline
	synthetic & 101 & 5 & 3.7-83.9 \\
	ijcnn1 & 49990 & 22 & 2.61 \\ 
	w8a & 49749 & 300 & 9.79\\
\end{tabular}
	\label{table:stat}
	\vspace{-.25in}
\end{table}
We consider two classification datasets, w8a and ijcnn1 from \cite{CL2011}, each of which has two classes. 
For each of SGD, Prox-SVRG and SAGA, we compare the effect of different sampling methods. 
We report the value $F(w^t)$, reached by the three sampling versions of stochastic optimization algorithms above, as a function of number of iterations $t$.

The cost function $F(w)$ used here is $L_1$-penalized logistic regression  with $\lambda = 10^{-4}$, i.e., $\phi_i(w) = \log (1+\exp(-y_i\langle x_i,w\rangle))$ and $r(w) = \|w\|_1$.
Each experiment is run for $T = 30n$ iterations and repeated $k = 100$ times. 
In all experiments, the step sizes $\gamma$ are 1, except the experiments for Prox\_SVRG, for which larger step size 2 is used.
The results are depicted in Figure~\ref{fig:main}. 
Again, the stochastic optimization algorithms with MABS are consistently the best among the algorithms. Comparing the results for the datasets \textit{ijcnn1} (with $\tau=2.61$) and \textit{w8a} (with $\tau=9.79$), MABS is more helpful for \textit{w8a}. This confirms the intuition that MABS improves the convergence rate more for a dataset with larger $\tau$ (see Figure~\ref{fig:saga_w8a} and \ref{fig:saga_ijcnn1}). In Figure~\ref{fig:Prox_ijcnn1} and \ref{fig:Prox_w8a}, the results for different sampling methods are similar to each other, this might be due to of the fact that Prox\_SVRG has a variance reduction technique \cite{AY2016} which is more efficient here than non-uniform sampling technique. Whereas, in Figure~\ref{fig:saga_w8a} MABS is still efficient in improving the convergence rate for SAGA, that has a varaince reduction technique \cite{DBL2014}. 
We also tested MABS in conjunction with S2GD and  Quasi\_Newton (with step size 0.0001).
For S2GD we use the algorithm from \cite{KR2013} with step size 1. S2GD\_MABS is 10 times closer to the optimal value than S2GD with uniform sampling.
For Quasi\_Newton we use the algorithm from \cite{BHNS2016} with step size 0.0001 and $M=200$. Quasi\_Newton\_MABS is 13.6 closer to the optimal value than Quasi\_Newton.

\begin{figure}[t] 
	\centering
	\subcaptionbox{ SGD on w8a dataset.}[.3\linewidth][c]{%
		\includegraphics[width=.3\linewidth]{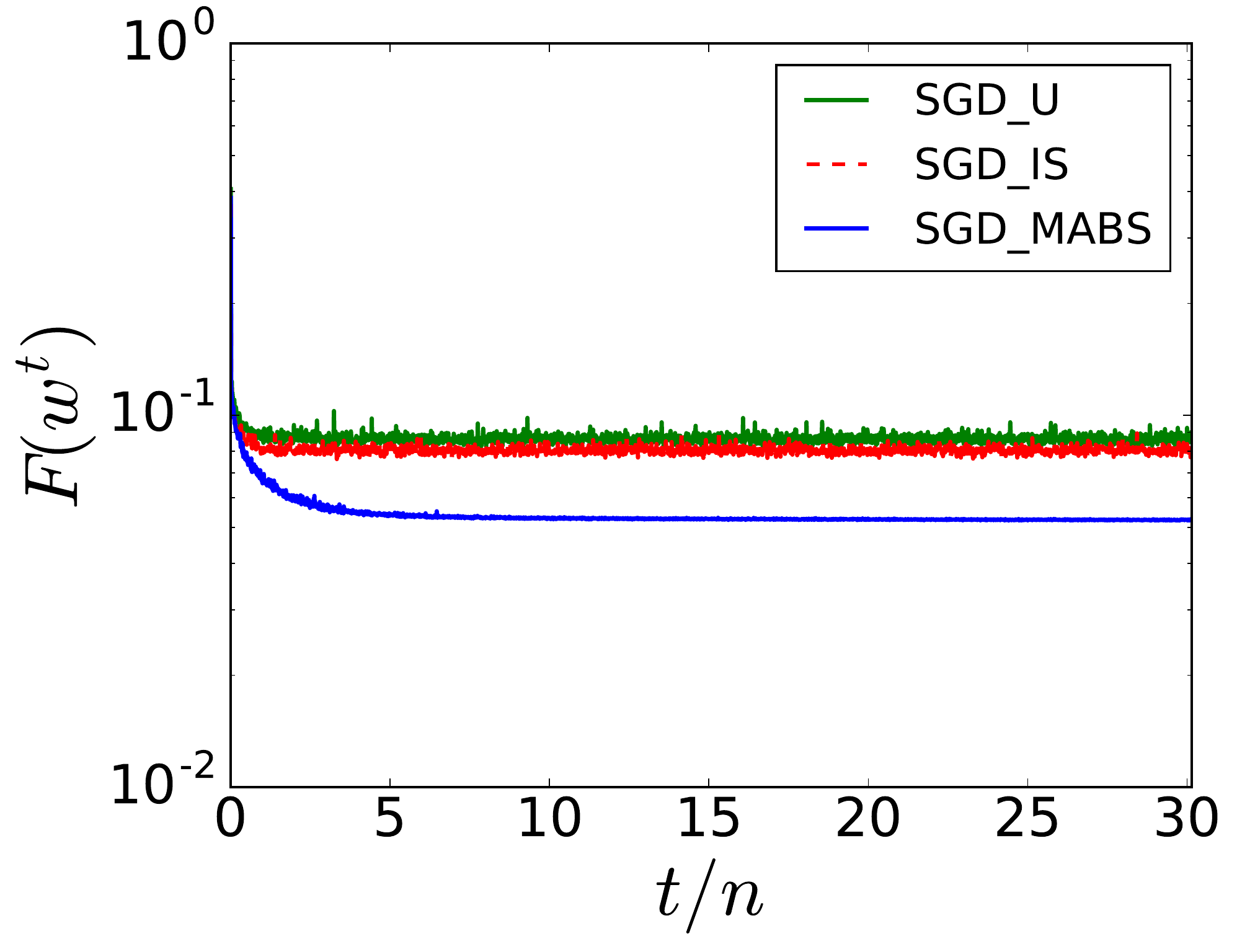}}
	\subcaptionbox{ Prox-SVRG on w8a dataset.\label{fig:Prox_w8a}}[.3\linewidth][c]{%
		\includegraphics[width=.3\linewidth]{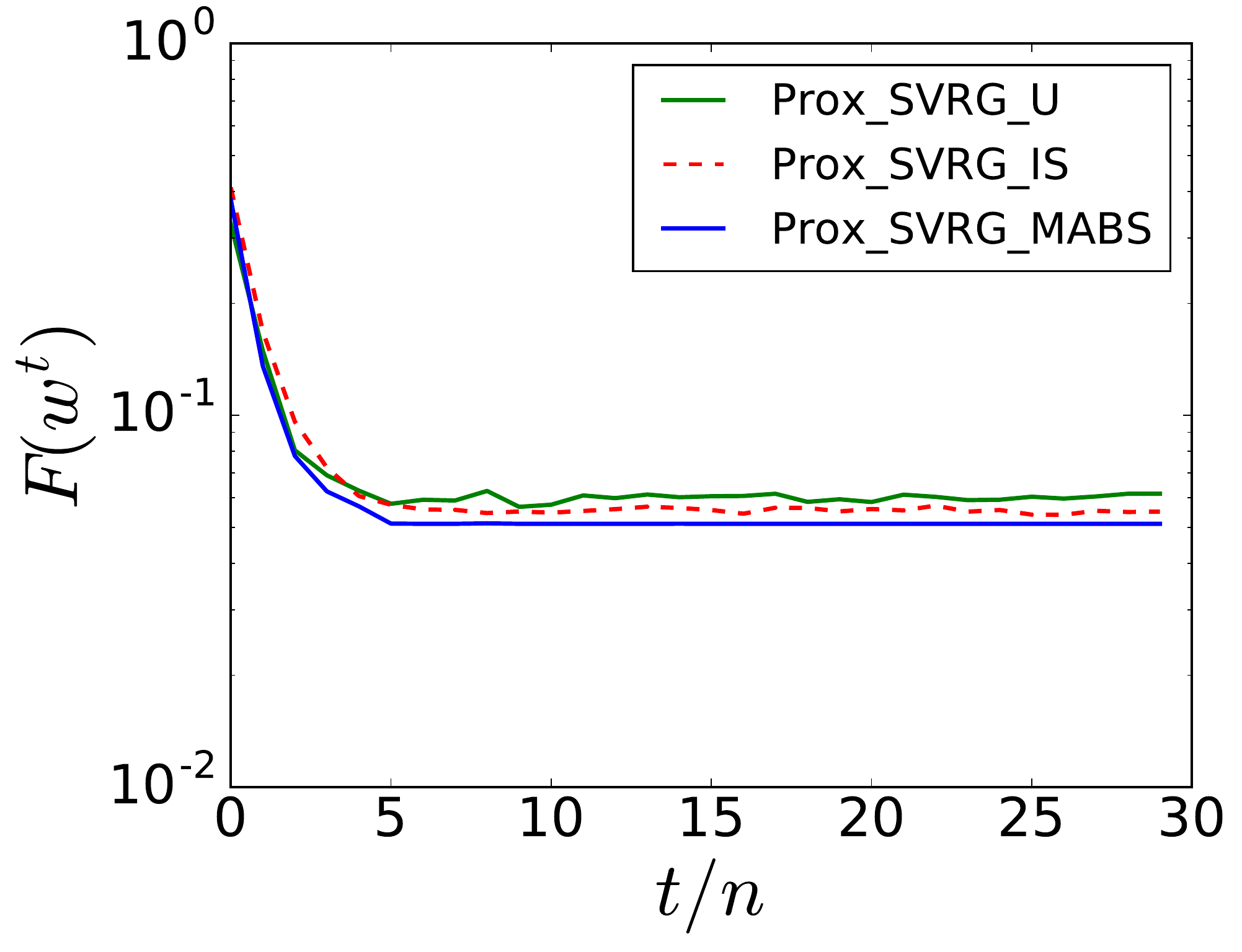}}\quad
	\subcaptionbox{ SAGA on w8a dataset. \label{fig:saga_w8a}}[.3\linewidth][c]{%
		\includegraphics[width=.3\linewidth]{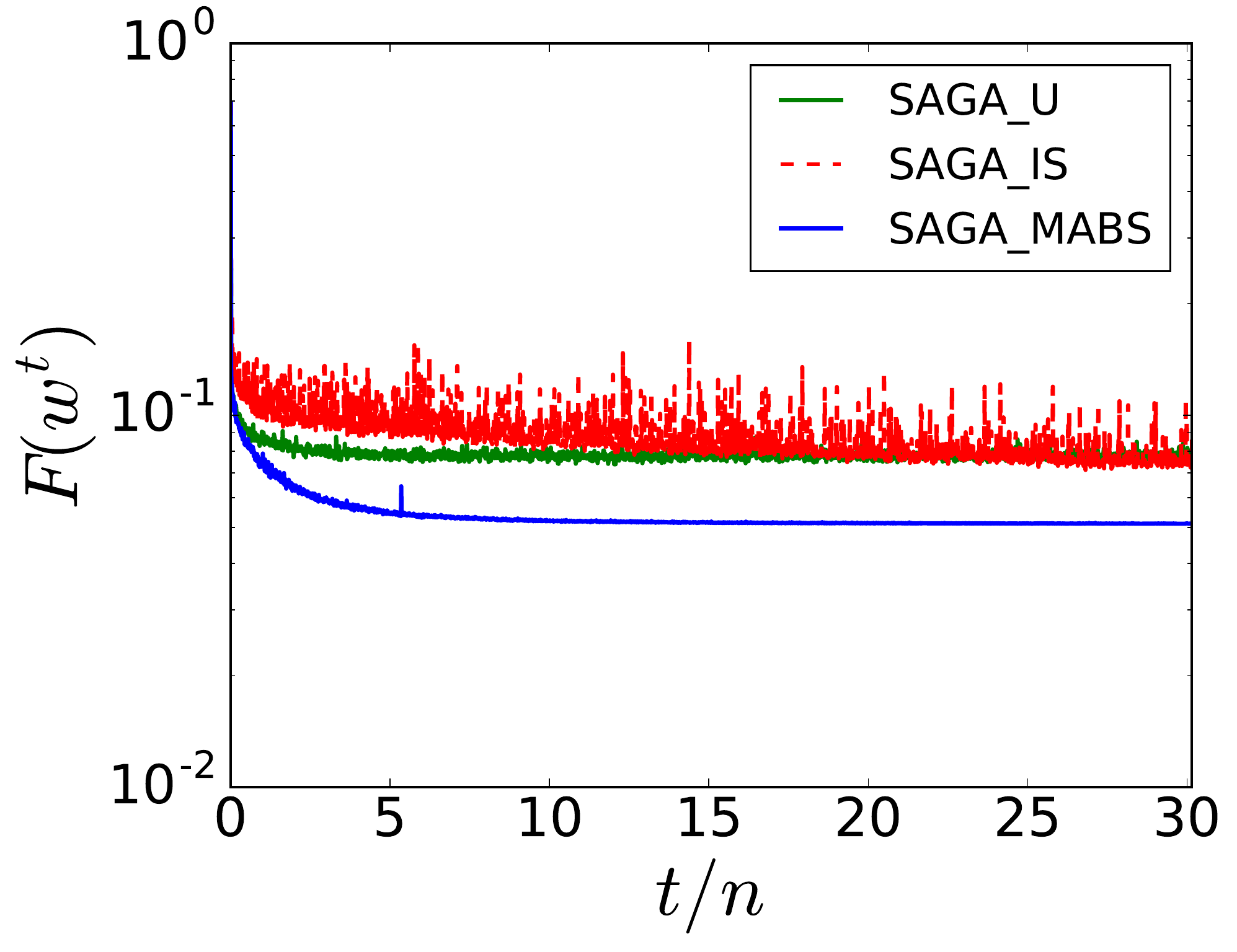}}
	\vspace{-.1in}
	\bigskip
		\subcaptionbox{ SGD on ijcnn1 dataset.}[.3\linewidth][c]{%
		\includegraphics[width=.3\linewidth]{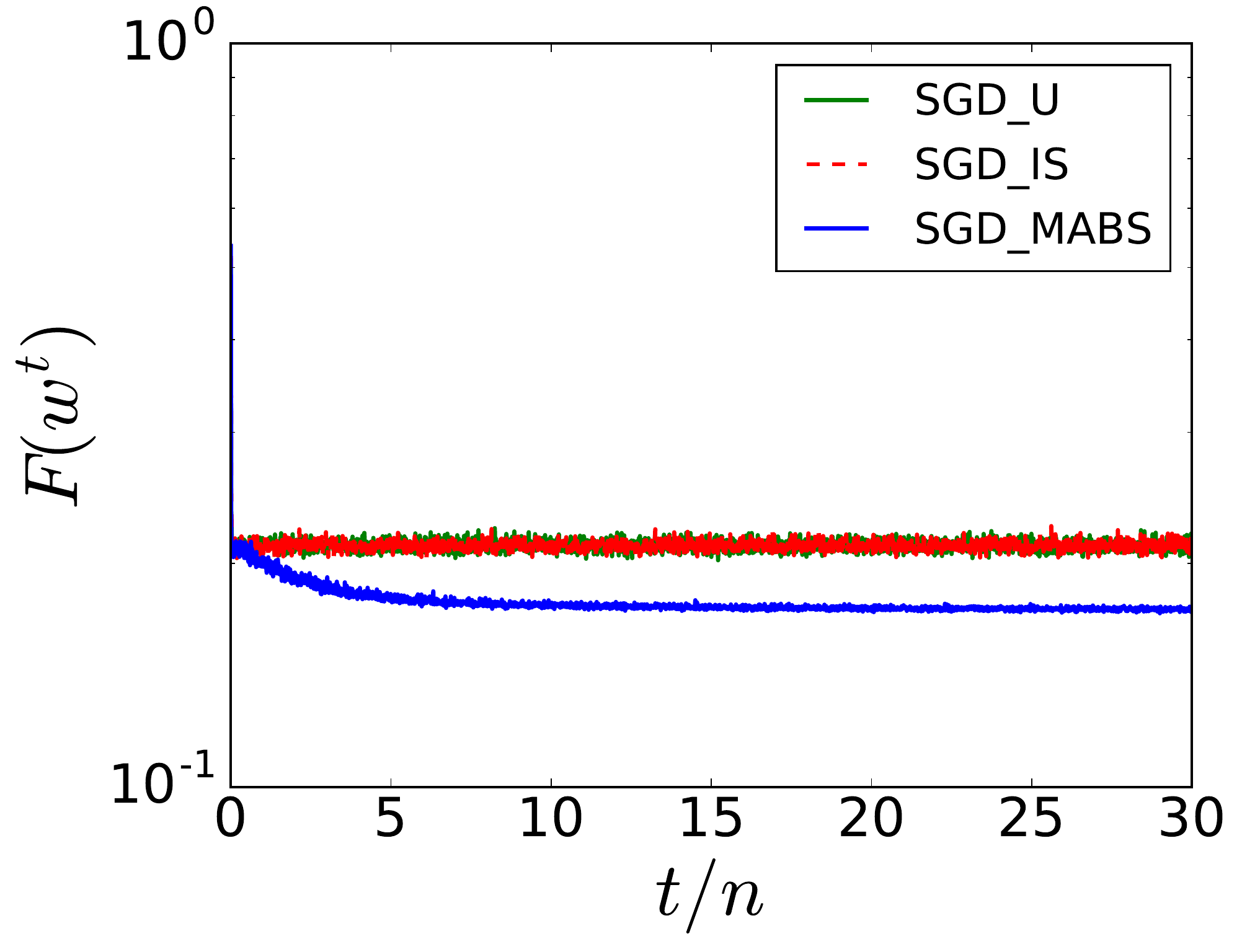}}\quad
		\subcaptionbox{ Prox-SVRG on ijcnn1 dataset.\label{fig:Prox_ijcnn1}}[.3\linewidth][c]{%
			\includegraphics[width=.3\linewidth]{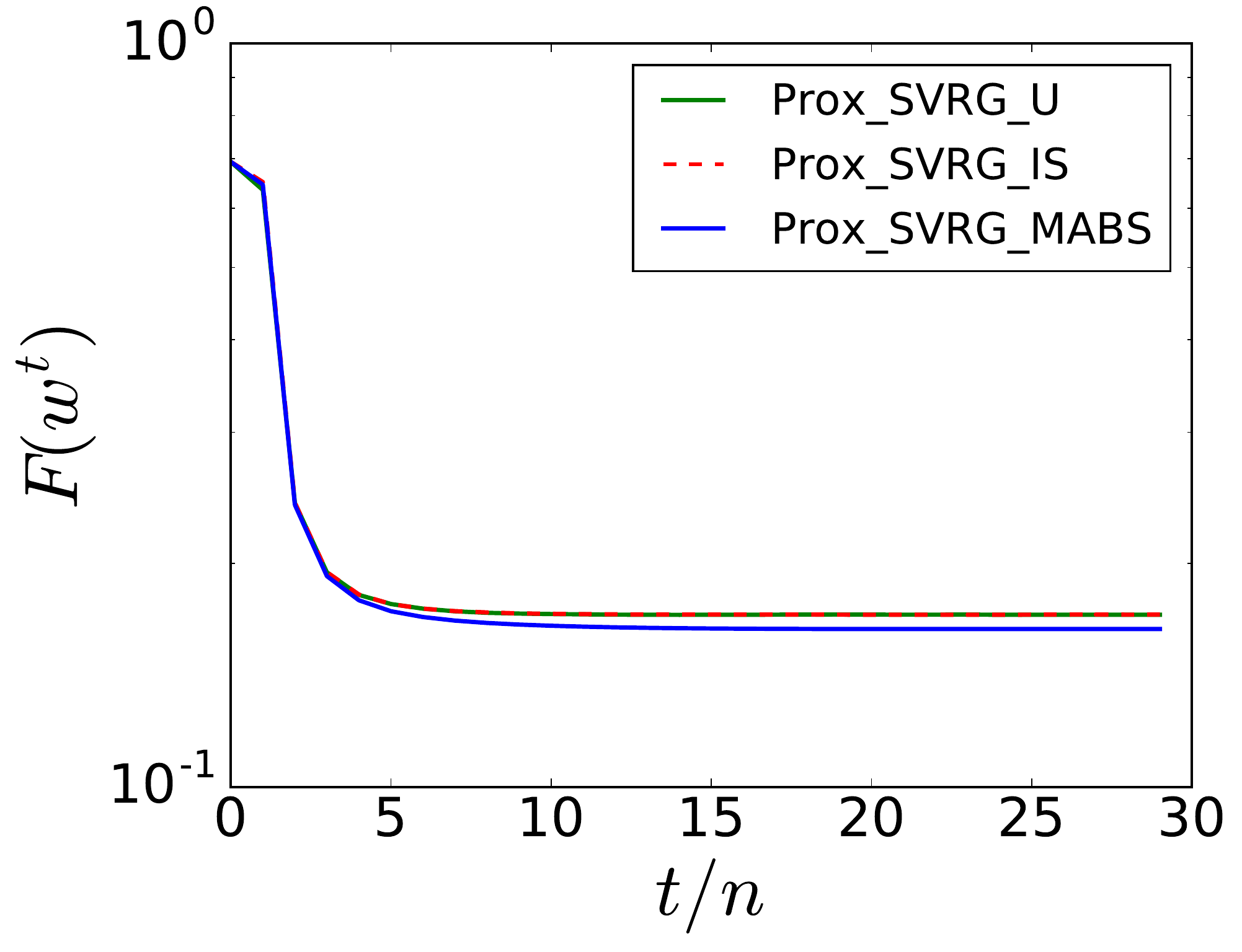}}\quad
	\subcaptionbox{ SAGA on ijcnn1 datset.\label{fig:saga_ijcnn1}}[.3\linewidth][c]{%
		\includegraphics[width=.3\linewidth]{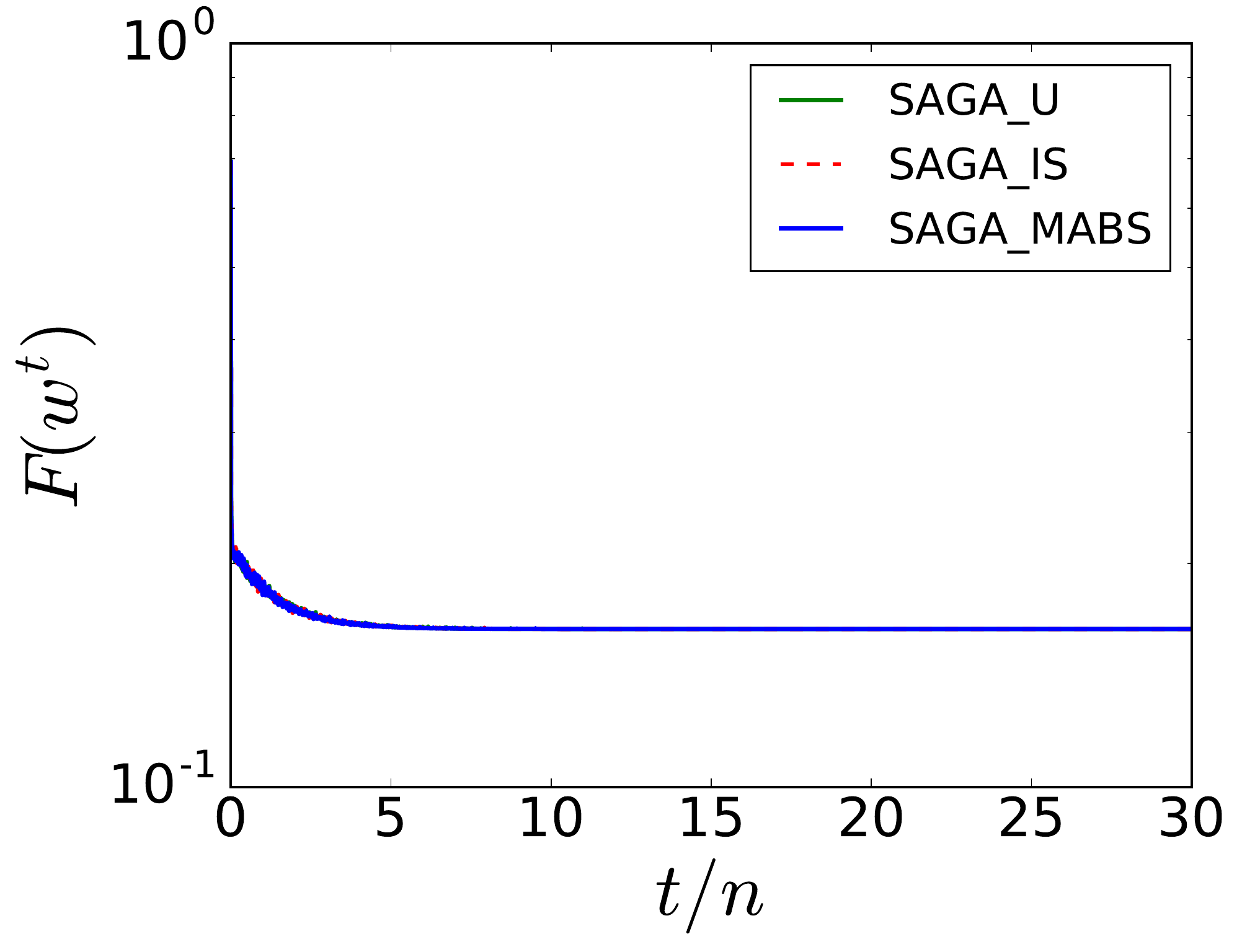}}
		\vspace{-.2in}
	\caption{Comparison of three different stochastic optimization algorithms (SGD, Prox-SVRG and SAGA) on two datasets (w8a and ijcnn1) when using different sampling methods. MABS is never suboptimal and often significantly outperforms the other sampling methods.\vspace{-.1in}}\label{fig:main}
\end{figure}

\vspace{-.1in}
\subsection{Stability} \label{sec:exp_robust}
\vspace{-.1in}
Following the discussion in Section~\ref{sec:applications}, we study the robustness of SGD, Prox-SVRG and SAGA when using  a large step size. 
In particular, we consider the \textit{w8a} dataset 
and $L_1$-penalized logistic regression as above.
We collect the value $F(w^T)$ at final iteration $T=60n$ for different stochastic optimization algorithms in conjunction with different sampling methods and fixed step size $\gamma$. Each experiment is repeated $k=50$ times. 
The results are depicted in Figure~\ref{fig:gamma} and show that MABS is indeed a more robust sampling method; 
SGD\_MABS is able to find the optimal coordinate up to $\gamma=5$ (see Figure~\ref{fig:gamma_SGD_w8a}), whereas SGD and SGD\_IS diverge after $\gamma = 0.5$. In Figure~\ref{fig:gamma_SVRG_w8a}, the difference between three sampling methods is less but still Prox\_SVRG\_MABS outperforms the others. SAGA\_MABS is also more robust than SAGA with other sampling methods, it is able to find the optimal coordinate up to $\gamma=3$ and diverges after that (see Figure~\ref{fig:gamma_SAGA_w8a}).

\begin{figure}[t] 
	\centering
	\subcaptionbox{ SGD on w8a dataset.\label{fig:gamma_SGD_w8a}}[.3\linewidth][c]{%
		\includegraphics[width=.3\linewidth]{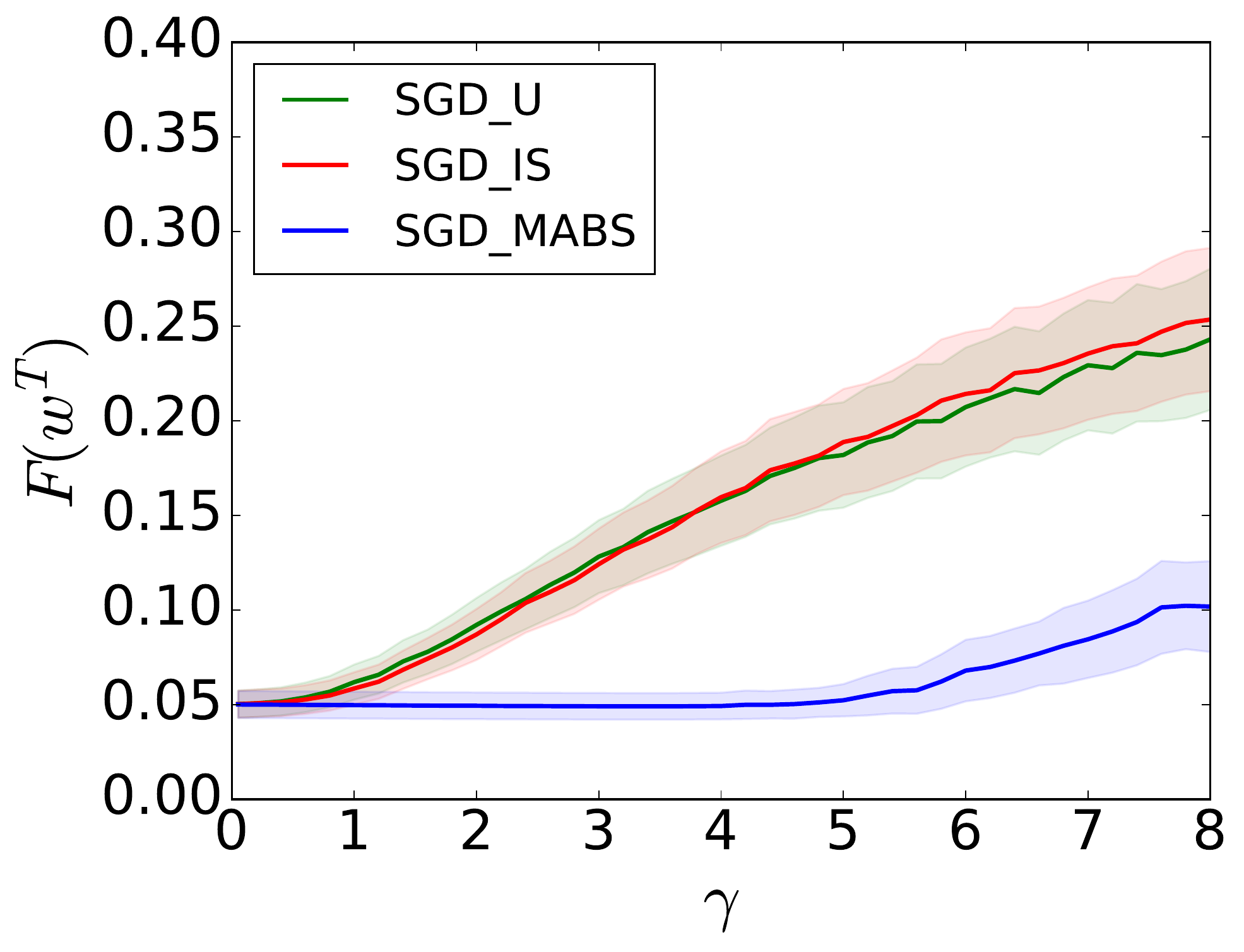}}
	\subcaptionbox{ Prox-SVRG on w8a dataset.\label{fig:gamma_SVRG_w8a}}[.3\linewidth][c]{%
		\includegraphics[width=.3\linewidth]{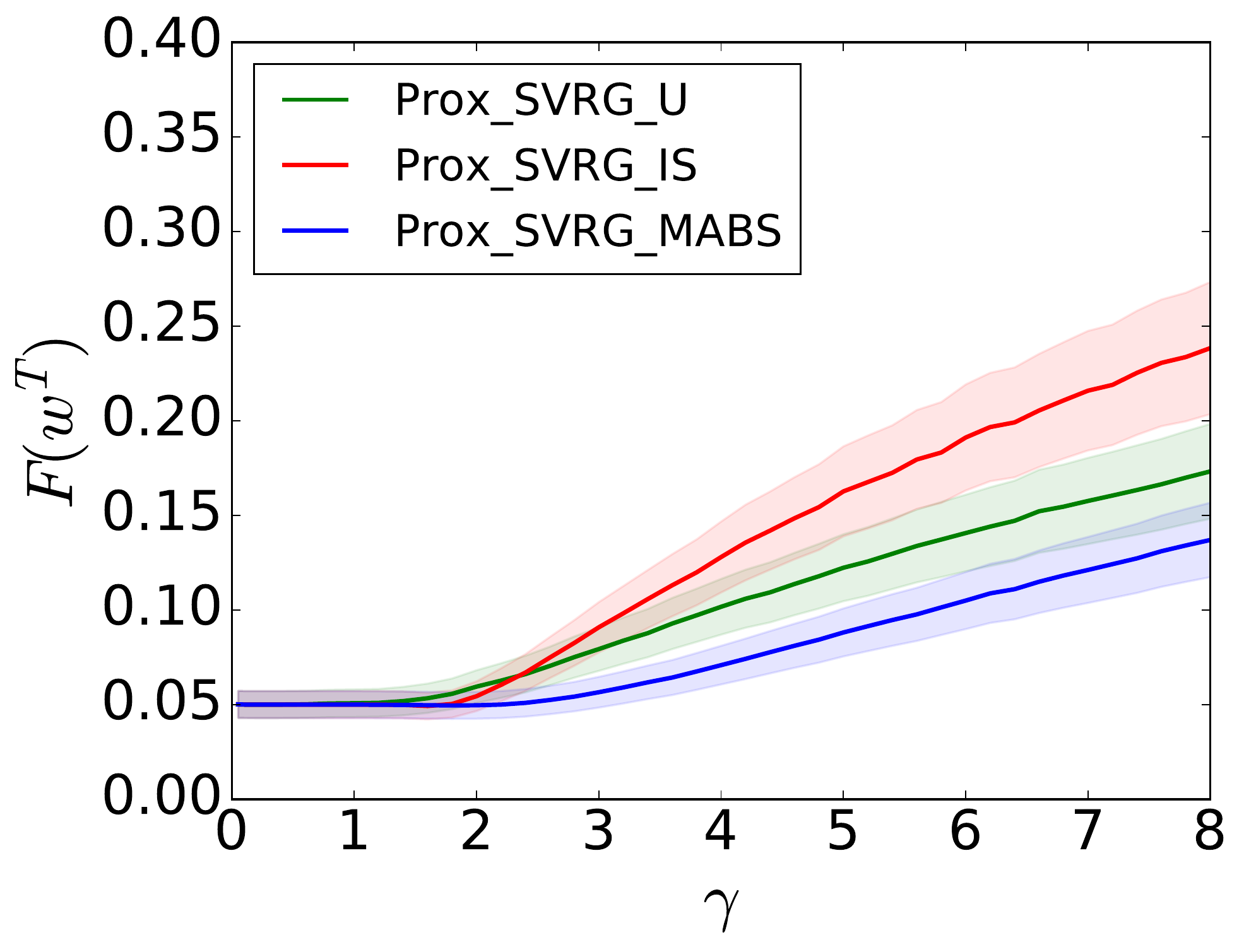}}\quad
	\subcaptionbox{ SAGA on w8a dataset. \label{fig:gamma_SAGA_w8a}}[.3\linewidth][c]{%
		\includegraphics[width=.3\linewidth]{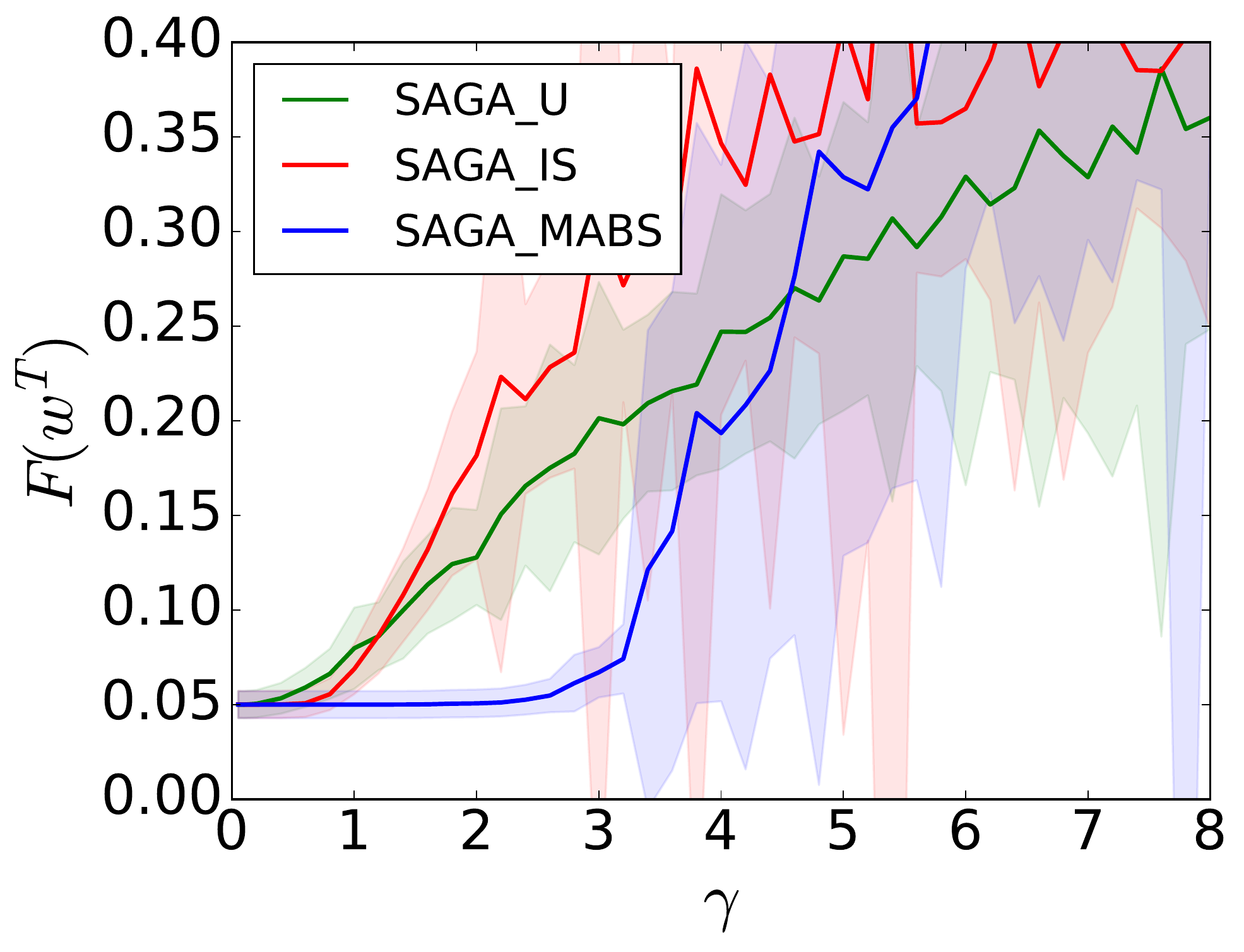}}
	\caption{Comparison of three different stochastic optimization algorithms (SGD, Prox-SVRG and SAGA) on w8a when using different sampling methods and different step sizes $\gamma$. MABS significantly outperforms the other methods and is able to find the optimal value even for a large $\gamma$. \vspace{-.2in}}\label{fig:gamma}
\end{figure}

\vspace{-.1in}
\subsection{Training time}
\vspace{-.1in}
We briefly note that adding MABS does not cost much with respect to training time. For example, given high-dimensional data with $d=4000$ and $n=50000$, empirically, SGD\_MABS uses only 10\%  more clock-time than SGD. In contrast, SGD\_IS with $p\sim G_i$ uses 40\% more clock-time than SGD, and if $p \sim L_i$ is so slow (as calculating $L_i$ is very expensive) that that our simulations did not terminate. 

\vspace{-.1in}
\section{Conclusion and Future Work}
\vspace{-.1in}
In this work, a novel sampling method (called MABS) is presented to reduce the variance of gradient estimation. The method is inspired by multi-armed bandit algorithms (in particular EXP3) and does not require any preprocessing. First, the variance of the unbiased estimator of the gradient at iteration $t$ is defined as a function of the sampling distribution $p^t$ and of the gradients of sub-cost functions $\nabla \phi_{i}(w^t)$. Next, considering the past information, MABS minimizes this cost function by appropriately updating to $p^t$,
and learns the optimal distribution $p^\star$ given the set of selected  datapoints $\{i_t\}_{1\leq t \leq T}$ and gradients $\{\nabla \phi_{i_t}(w^t)\}_{1\leq t \leq T}$. 
It is shown that under a natural assumption (bounded gradients) MABS can asymptotically approximate the optimal variance within a factor of 3. Moreover, MABS combined with three stochastic optimization algorithms  (SGD, Prox\_SVRG, and SAGA) is tested on real data. We observe its effectiveness on variance reduction and the rate of convergence of these optimization algorithms as compared to other sampling approaches. 
Furthermore,  MABS is tested on synthetic datasets, and its effectiveness is observed for a large range of $\tau$ (i.e., the ratio of maximum smoothness to the average smoothness). It is also observed that SGD\_MABS is significantly more stable than SGD with other sampling methods. Several important directions remain open. First, one would like to improve the constants in the bound in Theorem~\ref{thm:main}. Secondly,  although we observe robustness, finding the optimal step size $\gamma$ for Prox\_SVRG and SAGA remains open. Lastly, it could be of interest to extend the work for other stochastic optimization methods, both by providing theoretical guarantees and observing their performance in practice.

\newpage

\bibliographystyle{plain}
\newpage
\bibliography{SPReferences.bib}

\newpage

\appendix
\section{Appendix} \label{sec:appendix}

\textbf{Lemma~\ref{lem:bound}} 
	For any real value constant $\zeta \leq 1$ and any valid distributions $p^t$ and $p^\star$ we have
	\begin{equation}\label{eq:transform_A}
	(1-2\zeta) \mathbb{V}_e^t\left(p^t\right) - (1-\zeta) \mathbb{V}_e^t\left(p^\star\right) \leq \langle p^t - p^\star, \nabla \mathbb{V}_e^t\left(p^t\right) \rangle + \zeta \langle p^\star,\nabla \mathbb{V}_e^t\left(p^t\right) \rangle.
	\end{equation}

\begin{proof}
	
	The function $\mathbb{V}_e^t\left(p\right)$ is convex with respect to $p$, hence for any two  $p^t$ and $p^\star$ we have
	\begin{equation}
	\mathbb{V}_e^t(p^t) - \mathbb{V}_e^t(p^\star) \leq \langle p^t-p^\star, \nabla \mathbb{V}_e^t(p^t)  \rangle.
	\end{equation}
	Multiplying both sides of this inequality by $1-\zeta$, and noting that \eqref{eq:Bandit} yields that $\langle p^t, \nabla \mathbb{V}_e^t(p^t)  \rangle = - \sum_{i=1}^{n}p^t_i \frac{a^t_i}{(p^t_i)^2}= -\mathbb{V}_e^t(p^t)$ concludes the proof.
	\begin{equation*}
	\begin{aligned}
	(1-\zeta) \mathbb{V}_e^t(p^t) - (1-\zeta) \mathbb{V}_e^t(p^\star) &\leq  \langle p^t-p^\star, \nabla \mathbb{V}_e^t(p^t)  \rangle -\zeta \langle p^t-p^\star, \nabla \mathbb{V}_e^t(p^t)  \rangle \\& =
	\langle p^t-p^\star, \nabla \mathbb{V}_e^t(p^t)  \rangle +\zeta \langle p^\star, \nabla \mathbb{V}_e^t(p^t)  \rangle + \zeta \mathbb{V}_e^t(p^t).
	\end{aligned}
	\end{equation*}
\end{proof}

\textbf{Theorem~\ref{thm:main}}
	Let $T \geq 25n\ln n \cdot \max_i(a_i)^2/(4 \overline{(a^2)})$. 
	Using Algorithm~\ref{alg:Bandit} with $\eta = 0.4$ and $\delta=\sqrt{\eta^4\ln n/(Tn^5\overline{(a^2)})}$ to minimize \eqref{eq:Bandit} with respect to $\{p^t\}_{1\leq t \leq T}$, we have
	\begin{equation}\label{eq:performance1_A}
	\sum_{t=1}^{T} \mathbb{V}_e^t(p^t)  \leq 3 \sum_{t=1}^{T} \mathbb{V}_e^t(p^\star) +
	50 \sqrt{n^5T\overline{(a^2)} \ln n} ,
	\end{equation}
	where $a_i \geq \sup_t \{a^t_i\}$ is an upper bound for $a_i^t$, and $\overline{(a^2)} = \sum_{i=1}^{n}a_i^2/n$.

The condition $T \geq 25n\ln n \cdot \max_i(a_i)^2/(4 \overline{(a^2)})$ ensures that $\delta \hat{r}^t_i \leq 1$, which we need in the proof. 
\begin{proof}
	The proof uses same approach as the proofs in the multiplicative-weight update algorithms (see for example \cite{ACFS2002}), where we adapt it by using Lemma~\ref{lem:bound}.
	The proof is based on upper bounding and lower bounding the potential function $W^T$ at final iteration $T$. 
	Let $r_i^t = a^t_i/(p^t_i)^2$ be the reward of datapoint $i$ and $\hat{r}^t_i  =r_i^t*1_{\{I_t=i\}}/p^t_i$ be an unbiased estimator for $r_i^t$. Then, the update rule of weight $\text{w}_i^t$ is $\text{w}_i^{t+1}=\text{w}_i^t\cdot\exp(\delta\hat{r}^t_i)$.
	Therefore, $\text{w}^T_i=\exp\left(\delta\sum_{t=1}^{T}\hat{r}^t_i\right)$  and the potential function $W^T= \sum_{i=1}^{n} \text{w}^T_i \geq  \text{w}^t_j=\exp\left(\delta\sum_{t=1}^{T}\hat{r}^t_j\right)$ for all $1\leq j \leq n$. Knowing that $W^1=\sum_{i=1}^{n} \text{w}^1_i=n$,
	we get the following lower bound for the potential function $W^T$,
	\begin{equation}\label{eq:lower_bound}
	\delta \sum_{t=1}^{T} \hat{r}^t_j - \ln n \leq \ln\frac{W^T}{W^1}.
	\end{equation}
	Now, let us upper bound $W^T$.
	\begin{equation}
	\frac{W^{t+1}}{W^t} = \frac{\sum_{i=1}^{n}\text{w}^{t+1}_i}{W^t} = \frac{\sum_{i=1}^{n}\text{w}^{t}_i \e^{\delta \hat{r}^t_i}}{W^t} = \sum_{i=1}^{n} \left(\frac{p^t_i-\eta/n}{1-\eta}\right) \e^{\delta \hat{r}^t_i}.
	\end{equation}
	Using the inequality $\e^x < 1+x+x^2$ (for $x<1$), we have
	\begin{equation} \label{eq:x_leq}
	\begin{aligned}
	\frac{W^{t+1}}{W^t} \leq \sum_{i=1}^{n} \left(\frac{p^t_i-\eta/n}{1-\eta}\right) \left(1+\delta \hat{r}^t_i+(\delta \hat{r}^t_i)^2\right) &\leq 1+ \frac{\delta}{1-\eta}\sum_{i=1}^{n}p^t_i \hat{r}^t_i + \frac{\delta^2}{1-\eta}\sum_{i=1}^{n}p^t_i (\hat{r}^t_i)^2. 
	\end{aligned}
	\end{equation} 
	Using the inequality $\ln (1+x)\leq x$ which holds for all $x\geq
	0$ we get
	\begin{equation} \label{eq:upper1}
	\ln \frac{W^{t+1}}{W^t} \leq \frac{\delta}{1-\eta} \sum_{i=1}^{n}p^t_i \hat{r}^t_i + \frac{\delta^2}{1-\eta}\sum_{i=1}^{n}p^t_i (\hat{r}^t_i)^2.
	\end{equation}
	If we sum \eqref{eq:upper1} for $1\leq t \leq T$, we get the following telescopic sum 
	\begin{equation} \label{eq:upper_bound1}
	\ln \frac{W^{T}}{W^1} =\sum_{t=1}^{T}\ln \frac{W^{t+1}}{W^t}   \leq \frac{\delta}{1-\eta} \sum_{t=1}^{T}\sum_{i=1}^{n}p^t_i \hat{r}^t_i + \frac{\delta^2}{1-\eta} \sum_{t=1}^{T} \sum_{i=1}^{n}p^t_i (\hat{r}^t_i)^2.
	\end{equation}
	
	Combining the lower bound \eqref{eq:lower_bound} and the upper bound \eqref{eq:upper_bound1}, we get
	
	\begin{equation} \label{eq:difference1}
	\delta\sum_{t=1}^{T} \hat{r}^t_j - \ln n \leq \frac{\delta}{1-\eta} \sum_{t=1}^{T}\sum_{i=1}^{n}p^t_i \hat{r}^t_i + \frac{\delta^2}{1-\eta}\sum_{t=1}^{T}\sum_{i=1}^{n}p^t_i (\hat{r}^t_i)^2.
	\end{equation}
	Given $p^t$ we have $\E[(\hat{r}^t_i)^2]  = (r^t_i)^2/p^t_i$, 
	%
	hence, taking expectation of ~\eqref{eq:difference1} yields that
	\begin{equation}\label{eq:difference2}
	\delta \sum_{t=1}^{T} r^t_j - \ln n \leq \frac{\delta}{1-\eta} \sum_{t=1}^{T}\sum_{i=1}^{n}p^t_i r^t_i + \frac{\delta^2}{1-\eta}\sum_{t=1}^{T}\sum_{i=1}^{n}(r^t_i)^2.
	\end{equation}
	By multiplying \eqref{eq:difference2} by $p^\star_j$ and summing over $j$, we get
	\begin{equation} \label{eq:full_bound}
	\delta\sum_{t=1}^{T} \sum_{j=1}^{n} p^\star_jr^t_j - \ln n \leq \frac{\delta}{1-\eta} \sum_{t=1}^{T}\sum_{i=1}^{n}p^t_i r^t_i + \frac{\delta^2}{1-\eta}\sum_{t=1}^{T}\sum_{i=1}^{n}(r^t_i)^2,
	\end{equation}
	As $r^t_i= a^t_i/(p^t_i)^2=-\nabla_i \mathbb{V}_e^t(p^t)$, we have $\sum_{i=1}^{n} p_ir^t_i = - \langle p, \nabla \mathbb{V}_e^t(p^t) \rangle$ for any distribution $p$, by plugging this in \eqref{eq:full_bound} and rearranging it, we find
	\begin{equation} \label{eq:difference3}
	\sum_{t=1}^{T} \langle p^t-p^\star, \mathbb{V}_e^t(p^t) \rangle + \eta \sum_{t=1}^{T} \langle p^\star, \mathbb{V}_e^t(p^t) \rangle   \leq \frac{1-\eta}{\delta}\ln n +  \delta \sum_{t=1}^{T} \sum_{i=1}^{n}(r^t_i)^2.
	\end{equation}
	Using Lemma~\ref{lem:bound} with $\zeta = \eta$ in \eqref{eq:transform}, we have
	\begin{equation} \label{eq:whicheta}
	(1-2\eta) \sum_{t=1}^{T} \mathbb{V}_e^t(p^t) - (1-\eta) \sum_{t=1}^{T} \mathbb{V}_e^t(p^\star) \leq
	\frac{1-\eta}{\delta}\ln n +  \delta \sum_{t=1}^{T} \sum_{i=1}^{n}(r^t_i)^2,
	\end{equation}
	which yields
	\begin{equation} \label{eq:upper_bound}
	\sum_{t=1}^{T} \mathbb{V}_e^t(p^t)  \leq \frac{1-\eta}{1-2\eta}  \sum_{t=1}^{T} \mathbb{V}_e^t(p^\star) +
	\frac{1-\eta}{\delta (1-2\eta)}\ln n +  \frac{\delta}{1-2\eta} \sum_{t=1}^{T} \sum_{i=1}^{n}(r^t_i)^2.
	\end{equation}
	Note that \eqref{eq:whicheta} gives an upper bound on $\sum_{t=1}^{T} \mathbb{V}_e^t(p^t)$ only if $\eta < 0.5$.
	Finally, we know that $r^t_i \leq n^2a^t_i/\eta^2$. By setting $\eta = 0.4$ and $\delta=\sqrt{\eta^4\ln n/(Tn^5\overline{(a^2)})}$, we conclude the first part of proof 
	\begin{equation}
	\sum_{t=1}^{T} \mathbb{V}_e^t(p^t)  \leq 3 \sum_{t=1}^{T} \mathbb{V}_e^t(p^\star) +
	50 \sqrt{n^5T\overline{(a^2)} \ln n}  .
	\end{equation}
	With a tree structure (similar to the interval tree), we can update $\text{w}_{i_t}$ and sample from $p^t$ in $O(\log n)$. The 
	
\end{proof}

\textbf{Computational complexity of MABS1:} Similar to IS,
MABS1 requires a memory of size $O(n)$ to sore the weights $\text{w}^t_i$. At each iteration $t$, the weight $\text{w}_{i_t}$ is updated. If we want to update all the probabilities $p^t$, then each iteration of MABS1 needs $O(n)$ computations, which is expensive. However, with a tree structure (similar to the interval tree), we can reduce the computational complexity of sampling and updating $\text{w}_{i_t}$ to $O(\log n)$. 

\begin{corollary}\label{thm:main_T}
	Using MABS with $\eta = 0.4$ and $\delta=\nicefrac{1}{c}\sqrt{\nicefrac{\eta^4\ln n}{(Tn^5\overline{(a^2)})}}$, for some $c>1$, to minimize \eqref{eq:Bandit} with respect to $\{p^t\}_{1\leq t \leq T}$, we have
	\begin{equation}\label{eq:performance_T}
	\sum_{t=1}^{T} \mathbb{V}_e^t(p^t)  \leq 3 \sum_{t=1}^{T} \mathbb{V}_e^t(p^\star) +
	(\frac{75}{4c}+\frac{125}{4c}) \sqrt{n^5T\overline{(a^2)} \ln n} ,
	\end{equation}
	where $T \geq 25n\ln n \cdot \nicefrac{\max_i(a_i)^2}{(4c^2 \overline{(a^2)})} $, for some $a_i \geq \sup_t \{a^t_i\}$, and where $\overline{(a^2)} = \sum_{i=1}^{n}\nicefrac{a_i^2}{n}$. The complexity of MABS is $O(T\log n)$.
\end{corollary}
\begin{proof}
	Following the same steps as Theorem~\ref{thm:main}, we have \eqref{eq:upper_bound}, where now by choosing a smaller $\delta$ we can decrease the minimum acceptable $T$. Recall that in \eqref{eq:x_leq} we need to have $\delta \hat{r}^t_i \leq 1$. By choosing $\delta=\nicefrac{1}{c}\sqrt{\nicefrac{\eta^4\ln n}{(Tn^5\overline{(a^2)})}}$ we need to have $T \geq 25n\ln n \cdot \nicefrac{\max_i(a_i)^2}{(4c^2 \overline{(a^2)})} $, which allows us to use a time $c^2$ smaller than before.
\end{proof}

\subsection{MABS with IS} \label{sec:MABS_IS}

Now, similar to IS, assume that we can compute the bounds $a_i=\sup_t \{a^t_i\}$ exactly, then we can refine the algorithm and improve the results.
The idea is that,
instead of mixing the distribution $\{\text{w}_i^t/W^t\}_{1\leq i \leq n}$ with a uniform distribution, we mix $\{\text{w}_i^t/W^t\}_{1\leq i \leq n}$ with distribution $\{a_i^{2/5}/(\sum_{j=1}^{n}a_j^{2/5})\}_{1\leq i \leq n}$, i.e.,  $p^t_i \propto (1-\eta)\text{w}^{t+1}_i+\eta a_i^{2/5}$ for all $i$ (see Algorithm \ref{alg:Bandit2}).


\begin{algorithm}[t]
	\caption{MABS2}
	\label{alg:Bandit2}
	\begin{algorithmic}[1]
		\STATE \textbf{input: }  $a_i$, \hspace{6cm} for all $1\leq i \leq n$
		\STATE \textbf{initialize: }  $\eta = 0.4$ and $\delta=\sqrt{T\eta^4\ln n /(n \overline{(a^{2/5})})^5}$
		\STATE \textbf{initialize: }  $q_i = |a_i|^{2/5}/(\sum_{j=1}^{n}|a_j|^{2/5})$, $\text{w}_i^1=1$, \hspace{0.5cm} for all $1\leq i \leq n$
		\STATE \textbf{initialize: }  $p_i^1 = (1-\eta)\cdot 1/n+\eta q_i$, \hspace{2.2cm} for all $1\leq i \leq n$
		
		\For{$t=1:T$}{
			sample $i\sim p^t$ \\
			$\text{w}^{t+1}_i = \text{w}^t_i * \exp(\frac{\delta a_i^t}{(p^t_i)^3})$ \\
			$\text{w}^{t+1}_j = \text{w}^t_j$, \hspace{5.8cm} for all $ j \neq i$\\
			$W^{t+1} = \sum_{j=1}^{n}\text{w}^{t+1}_j$ \\
			$p^{t+1}_j \leftarrow (1-\eta)\frac{\text{w}^{t+1}_j}{W^{t+1}}+\eta q_j$, \hspace{3.4cm} for all $1\leq j \leq n$ \\
		}
		
	\end{algorithmic}
\end{algorithm}
\begin{corollary}\label{thm:main2}
	Let $T\geq 25n\ln n \cdot \overline{(a^{2/5})}/(4 \cdot \min_i a_i^{2/5})$. 
	Using Algorithm~\ref{alg:Bandit2} with $\eta = 0.4$ and $\delta=\sqrt{T\eta^4\ln n /(n \overline{(a^{2/5})})^5}$ to minimize \eqref{eq:Bandit} with respect to $\{p^t\}_{1\leq t \leq T}$, we have
	
	\begin{equation} \label{eq:performance}
	\sum_{t=1}^{T} \mathbb{V}_e^t(p^t)  \leq 3 \sum_{t=1}^{T} \mathbb{V}_e^t(p^\star) +
	50 \sqrt{n^5T (\overline{a^{2/5}})^5 \ln n} ,
	\end{equation}
	where
	$a_i \geq \sup_t \{a^t_i\}$ is an upper bound for $a_i^t$, and $\overline{(a^{2/5})} = \sum_{i=1}^{n}a_i^{2/5}/n$.
\end{corollary}
\begin{proof}
	Following the same steps as Theorem~\ref{thm:main}, we have \eqref{eq:upper_bound} where now, by knowing $a_i$ we can minimize the upper bound of $\sum_{i=1}^{n}(r^t_i)^2$ in it.
	\begin{equation} \label{eq:new_bound}
	\sum_{i=1}^{n}(r^t_i)^2 = \sum_{i=1}^{n}\frac{(a^t_i)^2}{(p^t_i)^4} \leq \frac{1}{\eta^4}\sum_{i=1}^{n}\frac{(a_i)^2}{(q_i)^4}.
	\end{equation}
	The right-hand side of \eqref{eq:new_bound} reaches its minimum for $q_i = a_i^{2/5}/\sum_{j=1}^{n}a_j^{2/5}$, and it is
	\begin{equation}
	\sum_{i=1}^{n}(r^t_i)^2  \leq \frac{1}{\eta^4}(\sum_{i=1}^{n}(a_i)^{2/5})^5.
	\end{equation}
	Plugging this bound in \eqref{eq:upper_bound} with $\eta=0.4$ and $\delta = \sqrt{T\eta^4\ln n /(n \overline{a^{0.4}})^5  }$ concludes the proof.
\end{proof}

In the same line of reasoning as in Section~\ref{sec:Introduction}, MABS2 can reduce the second term of the right-hand side of \eqref{eq:performance1} by $n^2$ in extreme cases (that happens when one of the $a_i$ is very large compared to the rest), but this requires computing $a_i$, which can be expensive and inefficient. 

\begin{remark} \label{rmk:adaptive_p}
	Note that the above results are derived for the case where we want to find an approximation of the optimal solution with fixed optimal $p^\star$, i.e., $\min_{p} \sum_{t=1}^{T} \mathbb{V}_e^t(p)$. However, with some additional assumptions, we can improve the results because we can perform close to the optimal solution with optimal $(p^{t})^\star$ for each iteration $t$, i.e., $\min_{p^t} \sum_{t=1}^{T} \mathbb{V}_e^t(p^t)$. These assumptions are bounds on the variation of $a_i^t$ over $t$. The new algorithm parallels Algorithm~\ref{alg:Bandit} with a resetting phase, where we reset the weights $\text{w}_i$ after some number of iterations. More precisely, the time $T$ is divided into bins. In the beginning of each bin, we reset $\text{w}_i$, then we run Algorithm~\ref{alg:Bandit}. The size of each bin is chosen such that the variation of $a_i^t$ across that bin is not large.
	Hence, for each bin we know that $\min_{p^t} \sum_{t=1}^{T} \mathbb{V}_e^t(p^t)$ is close to $\min_{p} \sum_{t=1}^{T} \mathbb{V}_e^t(p)$, 
	and, that by having an algorithm that performs close to $\min_{p} \sum_{t=1}^{T} \mathbb{V}_e^t(p)$ we can deduce that it also performs close to $\min_{p^t} \sum_{t=1}^{T} \mathbb{V}_e^t(p^t)$ (see \cite{BGZ2014} for more details).
\end{remark}

\textbf{Comparing MABS2 with IS}
First let us upper bound the effective-variance of IS for the general form \eqref{eq:Bandit},
\begin{equation} \label{eq:upper_bound_IS}
\sum_{t=1}^{T}\mathbb{V}^t_e(p^t) = \sum_{t=1}^{T}\sum_{i=1}^{n} \frac{a^t_i}{p_i^t} \leq \sum_{t=1}^{T}\sum_{i=1}^{n} \frac{a_i}{p_i^t}. 
\end{equation}
The right-hand side of \eqref{eq:upper_bound_IS} reaches its minimum for
 $p^t_i=p_i=\sqrt{a_i}/\left(\sum_{j=1}^{n}\sqrt{a_j}\right)$, 
\begin{equation} \label{eq:effective_IS}
\sum_{t=1}^{T}\mathbb{V}^t_e(p) \leq  T \left(\sum_{i=1}^{n}\sqrt{a_i}\right)^2.
\end{equation}
Now we have two upper bounds on the effective variance (\eqref{eq:performance} for MABS2 and \eqref{eq:effective_IS} for IS) that we want to compare.
According to \eqref{eq:optimal_p}, we know that $\sum_{t=1}^{T} \mathbb{V}_e^t(p^\star)$ is the optimal effective variance. Therefore, when we compare the effective variance of MABS2 with IS, we focus on the second term of \eqref{eq:performance}, i.e., we focus on the following term
\begin{equation} \label{eq:ratio_effective}
O\left(\frac{\sqrt{n^5T (\overline{a^{0.4}})^5 \ln n}}{T (\sum_{i=1}^{n}\sqrt{a_i})^2}\right).
\end{equation}
Similar to the discussion in Section~\ref{sec:Introduction}, we consider two extreme scenarios. 

(i) Let $a_1 = a_i$ for all $1\leq i\leq n$. Then \eqref{eq:ratio_effective} becomes $O\left(\sqrt{ n \ln n/T}\right)$.
(ii) Let $a_1$ be very large compared to others, i.e., $a_i/a_1 \to 0$ for $i>1$. Then \eqref{eq:ratio_effective} becomes 
$O\left(\sqrt{ \ln n/T}\right)$.

We see that the benefit of MABS2 exceeds that of IS when the ratio of number $n$ of datapoints to the number $T$ of iteration is small, and when the bounds $a_i$ on the magnitude of gradients $\|\nabla\phi_i(\cdot)\|$ are greatly varying.


\subsection{Definitions} \label{sec:definition}

\begin{definition}[$\bm{L}$-smooth] \label{def:smooth}
	Let $L>0$. Function $h(\cdot)$ is $L$-smooth if for any $x$ and $y \in \mathbb{R}^d$ 
	\begin{equation} \label{eq:smooth}
	h(y) \leq h(x) + \langle \nabla h(x),y-x \rangle + L\|x-y\|^2 .
	\end{equation}
\end{definition}
\begin{definition}[$\bm{\mu}$-strongly convex] \label{def:strong}
	Let $\mu>0$. Function $h(\cdot)$ is $\mu$-strongly convex if for any $x$ and $y \in \mathbb{R}^d$ 
	\begin{equation}\label{eq:strong}
	h(y) \geq h(x) + \langle \nabla h(x),y-x \rangle + \frac{\mu}{2}\|x-y\|^2 .
	\end{equation}
\end{definition}

\begin{definition}[Bregman divergence]
	Let $w_1$, $w_2\in \mathbb{R}^d$. The Bregman divergence associated with the function $\psi$ is 
	\begin{equation}
	\mathcal{B}_\psi(w_1,w_2) = \psi(w_1) - \psi(w_2) - \langle \nabla \psi(w_2),w_1-w_2\rangle.
	\end{equation}
\end{definition}
\begin{definition}[$\bm{\mu}$-strongly convex with respect to $\bm{\psi}$]
	Let $\mu >0$. Function $f(\cdot)$ is $\mu$-strongly convex with respect to a differentiable function $\psi$ if for any $w_1$ and $w_2 \in \mathbb{R}^d$
	\begin{equation}
	f(w_1) \geq f(w_2) + \langle \nabla \psi(w_2),w_1-w_2\rangle + \mu \mathcal{B}_\psi (w_1,w_2).
	\end{equation}
\end{definition}

\end{document}